\documentclass[lettersize,journal]{IEEEtran}
\usepackage{amsmath,amsfonts}
\usepackage{algorithmic}
\usepackage{algorithm}
\usepackage{array}
\usepackage{textcomp}
\usepackage{stfloats}
\usepackage{url}
\usepackage{verbatim}
\usepackage{graphicx}
\usepackage{cite}
\usepackage{bm}
\usepackage{amsthm}
\usepackage{subcaption}
\usepackage{tikz}
\usetikzlibrary{shapes,arrows,chains}
\usepackage{makecell}
\usepackage{multirow}
\newtheorem{theorem}{Theorem}
\newtheoremstyle{break}
  {\topsep}{\topsep}%
  {\itshape}{}%
  {\bfseries}{}%
  {\newline}{}%
\theoremstyle{break}
\newtheorem{definition}{Definition}

\begin{document}

\title{Multi-Task Learning with Multi-Task Optimization}

\author{Lu Bai, 
        Abhishek Gupta,
        and Yew-Soon Ong,~\IEEEmembership{Fellow,~IEEE}
\thanks{This paper was produced by the IEEE Publication Technology Group. They are in Piscataway, NJ.}
\thanks{Manuscript received April 19, 2021; revised August 16, 2021.}}

\markboth{Journal of \LaTeX\ Class Files,~Vol.~14, No.~8, August~2021}%
{Shell \MakeLowercase{\textit{et al.}}: A Sample Article Using IEEEtran.cls for IEEE Journals}

\IEEEpubid{0000--0000/00\$00.00~\copyright~2021 IEEE}

\maketitle

\begin{abstract}
Multi-task learning solves multiple correlated tasks. However, conflicts may exist between them. In such circumstances, a single solution can rarely optimize all the tasks, leading to performance trade-offs. To arrive at a set of optimized yet well-distributed models that collectively embody different trade-offs in one algorithmic pass, this paper proposes to view \emph{Pareto multi-task learning} through the lens of multi-task optimization. Multi-task learning is first cast as a multi-objective optimization problem, which is then decomposed into a diverse set of unconstrained scalar-valued subproblems. These subproblems are solved jointly using a novel multi-task gradient descent method, whose uniqueness lies in the iterative transfer of model parameters among the subproblems during the course of optimization. A theorem proving faster convergence through the inclusion of such transfers is presented. We investigate the proposed multi-task learning with multi-task optimization for solving various problem settings including image classification, scene understanding, and multi-target regression. Comprehensive experiments confirm that the proposed method significantly advances the state-of-the-art in discovering sets of Pareto-optimized models. Notably, on the large image dataset we tested on, namely NYUv2, the hypervolume convergence achieved by our method was found to be nearly two times faster than the next-best among the state-of-the-art. 
\end{abstract}


\section{Introduction}

\IEEEPARstart{M}{ulti}-task learning (MTL) solves multiple correlated tasks by jointly training them to improve generalization ability, either by shallow models \cite{zhang2021survey} or deep models \cite{ruder2017overview}. Many MTL approaches have been proposed in the past years and have shown superior performance in comparison to their single-task counterpart in various domains, such as computer vision \cite{liu2019end}, bioinformatics \cite{zhang2016deep}, speech and natural language processing \cite{collobert2008unified}, to name a few. Most MTL approaches are proposed to find one single solution, which usually minimizes a unified global function composed of all the task-specific objective functions. However, in many real-world applications, the tasks could conflict with each other. For example, as shown in \cite{kendall2018multi}, when learning two tasks (e.g., semantic segmentation and pixel-wise metric depth prediction) together with arbitrary weights of the loss functions, the performance of one or both could worsen compared to learning them independently. This highlights the existence of conflicts between tasks. In such circumstances, there may rarely be a single solution that is simultaneously performant on all tasks. 

Different from conventional MTL, multi-objective optimization (MOO) explicitly deals with problems bearing conflicting objectives (where improvements in one objective are accompanied by the worsening of another). MOO usually has a \emph{set} of Pareto optimal solutions, each embodying a different trade-off among the objectives \cite{ehrgott2005multicriteria}. Since in many cases preferred trade-offs are a priori unknown---and hence practitioners may be interested in viewing diverse solutions offering different trade-offs---it is often worthwhile to derive a representative subset of Pareto optimal solutions in MOO.

Given the above, \cite{sener2018multi} first formulated MTL as MOO, and extended the MGDA \cite{desideri2012multiple} to find a random Pareto optimal solution with less restriction. \cite{mahapatra2020multi} extended the work by developing an exact method to find the Pareto optimal solution satisfying user-specified preferences with respect to task-specific losses. Based on the multi-objective optimization method, \cite{lin2019pareto} and \cite{liu2021profiling} generalized the idea for finding a representative subset of well-distributed Pareto optimal solutions. As an alternative, researchers proposed methods to learn the whole Pareto front, either by generating locally continuous Pareto sets and Pareto fronts \cite{ma2020efficient}, training a hyper-network \cite{navon2021learning,lin2020controllable}, or incorporating preferences into the input space \cite{liu2021profiling}.

In the \emph{a posteriori} MOO literature, a common approach is to decompose the parent formulation into a collection of single-objective subproblems characterized by different reference directions---e.g., preference or weight vectors---in objective space \cite{zhang2007moea}. Similar subproblems, such as those scalarized by neighboring weight vectors, are intuitively expected to produce similar optima. In such settings, recent advances in gradient-free optimization have shown the potential to leverage latent relationships between distinct optimization tasks with different (but correlated) objective functions, with the goal of speeding up overall convergence rates \cite{bai2021multitask,luo2018evolutionary, bali2019multifactorial}. Algorithmic implementations of this idea, labeled as \emph{multi-task optimization} (MTO) \cite{swersky2013multi,gupta2015multifactorial, gupta2017insights}, point to a new box of tools with wide-ranging practical application, including in data science and machine learning pipelines, unmanned systems planning, complex design, among others \cite{gupta2022half}. It is therefore contended that even in the precinct of gradient-based learning, the transfer of information among MOO subproblems could substantially boost the search for Pareto MTL models.

\IEEEpubidadjcol

In this paper, we therefore propose a novel Pareto MTL algorithm from the multi-task optimization point of view, named \emph{Multi-Task Learning with Multi-Task Optimization} (MT$^2$O), that finds a representative set of Pareto optimized models in a single run. MT$^2$O first turns MTL into MOO, and then decomposes it into $N$ scalar unconstrained optimization subproblems. In contrast to \cite{lin2019pareto}, it solves these subproblems \emph{jointly} by a novel multi-task gradient descent method, where useful information is propagated among the subproblems through an iterative transfer mechanism. Our contributions in this paper are threefold:
\begin{itemize}
   \item We develop a new multi-task gradient descent method for MTL that can converge to a set of Pareto optimal models in one algorithmic pass. The models collectively embody multiple users' needs, beyond what can be captured by just a single model.
   \item The uniqueness of our approach lies in the iterative transfer of model parameters among MOO subproblems during the joint optimization run. Theoretical results showing faster convergence through the inclusion of such transfers are presented for the first time in the context of Pareto MTL. 
   \item Extensive experimental analysis and comparisons against several state-of-the-art baselines on various problem settings, including synthetic MOO problems, image classification, scene understanding (joint semantic segmentation, depth estimation, and surface normal estimation), and multi-target regression, confirm the efficacy and efficiency of the proposed MT$^2$O. 
\end{itemize} 

The rest of the paper is organized as follows. Section II introduces related works. Section III gives preliminaries used in this paper, including reframing MTL as MOO and  decomposition approaches to transform a MOO problem into a set of scalar optimization subproblems. The proposed MT$^2$O method for Pareto MTL with multi-task optimization is presented next in Section IV. In Section V, experiments on synthetic examples as well as real-world MTL datasets are conducted to rigorously evaluate performance. Finally, Section VI gives the conclusion.

\section{Background}

\subsection{Conflicts in Multi-Task Learning}
MTL solves multiple tasks together to improve learning efficiency and predictive accuracy, mainly by learning a shared representation from the data of related tasks \cite{zhang2021survey,ruder2017overview}. In practice, a unified objective function for MTL is constructed by weighting the empirical loss over all tasks through weights $w^i, i\in\{1,...,m\}$:
\begin{align}\label{mtl}
\theta^* = \text{argmin}_{\theta\in\mathbb R^d}\ \sum_{i=1}^m w^i L^i(\theta), 
\end{align}
where $L^i(\theta): \mathbb R^d\to \mathbb R$ is the loss of the $i$-th task, $d$ is the number of model parameters, and $m$ is the number of tasks. However, when the tasks are not highly related, which may be observed through conflicting or dominating gradients \cite{zhang2021survey,yu2020gradient,liu2021conflict,vandenhende2021multi}, directly optimizing \eqref{mtl} using gradient descent may significantly compromise the optimization of individual task's loss. 
To solve this problem, previous works have attempted to use random weighting methods \cite{lin2022reasonable}, adaptively re-weight the losses of each task \cite{kendall2018multi,ye2021multi}, or seek a better update vector by manipulating the task gradients \cite{chen2018gradnorm,sener2018multi,yu2020gradient,liu2021conflict}. However, prior works often lack convergence guarantees or return solutions with no controlled trade-offs between the conflicting tasks. As a result, these methods may optimize one loss while overlooking the others. When conflicts exist between task-specific objectives, no single solution can achieve the best performance across all tasks. Different trade-offs will produce different optimal solutions. As such, the multitude of possible trade-offs between the tasks can not be readily captured by traditional MTL methods.

\subsection{Multi-Objective Optimization}

MOO is the study of problems with conflicting objectives \cite{marler2004survey,miettinen2012nonlinear,deb2014multi}. It is usually formulated as follows:
\begin{align}\label{moo}
\theta^* = \text{argmin}_{\theta\in\mathbb R^d} {\bf L}(\theta) = (L^1(\theta), L^2(\theta),...,L^m(\theta)),
\end{align}
where ${\bf L}(\theta)$ is vector-valued with $m$ objectives. Given the existence of conflicts and (hence) trade-offs among the objectives, it produces a set of so-called Pareto optimal solutions. For completeness, we provide below the definitions of basic concepts of Pareto dominance, Pareto optimality, Pareto sets, and Pareto fronts, common in MOO \cite{zitzler1999multiobjective}:
\begin{definition}[Pareto Dominance]
A solution $\theta^a$ is said to Pareto dominate solution $\theta^b$ if $L^i(\theta^a)\leq L^i(\theta^b)$, $\forall i\in\{1,...,m\}$ and $\exists j\in\{1,...,m\}$ such that $L^j(\theta^a)<L^j(\theta^b)$.
\end{definition}
\begin{definition}[Pareto Optimality]
A solution $\theta^*$ is called Pareto optimal if there exists no solution $\theta$ that Pareto dominates $\theta^*$. 
\end{definition}
\begin{definition}[Pareto Set]
The set of all Pareto optimal solutions forms the Pareto set. 
\end{definition}
\begin{definition}[Pareto Front]
The image of the Pareto set in the objective function space is called the Pareto front.
\end{definition}

A variety of a posteriori methods for tackling MOO with population-based evolutionary algorithms have been proposed over the years, spanning dominance-based methods \cite{deb2000fast,yuan2015new}, indicator-based approaches \cite{falcon2020indicator}, and decomposition-based techniques \cite{zhang2007moea}. A distinctive advantage of evolution for problems of this kind is the implicit parallelism of its population, enabling the algorithm to sample and evaluate diverse regions of the search space, and hence converge simultaneously towards a representative subset of Pareto optimal solutions \cite{gupta2019back}. Note that typical gradient-based learning---which initializes and updates a single solution at a time---lacks such implicit parallelism. 
Of particular interest in this paper is the class of decomposition-based methods, popularized by Zhang and Li \cite{zhang2007moea}, where an MOO is first reduced to a number of scalar optimization subproblems. These can then be solved jointly by an evolving population of solutions, leading to relatively lower computational complexity than the independent optimization of each subproblem. However, in the case of large-scale MTL, the sampling and evaluation of a population of solutions can quickly become prohibitively expensive. Hence, in this paper, we envisage a novel multi-task gradient-descent method that can arrive at a diverse yet optimized set of MTL models in one algorithmic pass.

\subsection{Pareto Multi-Task Learning}

When there are conflicts among tasks in MTL, no single globally optimal model may exist. We thus aim to achieve a Pareto optimal model, whose performance vector lies in a preferred subregion of the Pareto front. Importantly, exact preference information may not be available beforehand. However, most traditional MTL methods (e.g., minimizing the weighted-sum of losses of all tasks) are restrictive in the sense that they can only produce solutions that map to convex parts of a possibly nonconvex front. To deal with this, the MTL was formulated as a MOO problem and the multiple gradient descent algorithm (MDGA) from \cite{desideri2012multiple} was extended to solve it in \cite{sener2018multi}. However, their method only finds a single Pareto optimal solution (i.e., parameters of a machine learning model), which may not precisely meet the needs of MTL practitioners. \cite{lin2019pareto} thus generalized the idea by decomposing the MOO problem into multiple multi-objective constrained subproblems with different trade-offs. They proposed a Pareto MTL algorithm to solve the subproblems independently, such that a set of well-distributed machine learning models with optimal trade-offs among the tasks could be obtained. Under specific preferences or priorities among the tasks, \cite{mahapatra2020multi} developed a method that can find a Pareto optimal solution satisfying the user-supplied preference with respect to task-specific losses. Given multiple preferences, a set of Pareto-optimized models was obtained. \cite{liu2021profiling} utilized Stein variational gradient descent (SVGD) to iteratively update a set of points towards the Pareto front while encouraging diversity among particles. As an alternative to the generation of finite and discrete solution sets, \cite{ma2020efficient} proposed a method that can generate locally continuous Pareto sets and Pareto fronts, providing a wider range of candidate solutions with varying trade-offs. To a similar end, \cite{navon2021learning} and \cite{lin2020controllable} proposed to train a hyper-network to learn the whole Pareto front, enabling controllable and on-demand generation of arbitrary numbers of candidate solutions. \cite{ruchte2021scalable} proposed a scalable multi-objective optimization method, which directly contextualizes network preferences into the input space to achieve well-distributed Pareto fronts.

\subsection{Multi-Task Optimization}
Multi-task optimization deals with solving multiple related optimization problems concurrently, often utilizing inter-task similarities to improve the optimization performance and efficiency. Suppose we have $N$ optimization tasks, each with an objective function parameterized by a task descriptor $\bm\lambda_i: f(\theta|\bm\lambda_i)\in\mathbb R$, $i\in\{1,...,N\}$. The descriptor $\bm\lambda_i$ specifies the parameters of task $i$. In addition, we assume that we have access to a meaningful similarity function quantifying the relationship between tasks $i$ and $j$ (e.g., the ordinal correlation between their objective functions \cite{zhou2018study}).
The workflow of multi-task optimization is then to jointly search for optimal values for each task:
\begin{align*}
 \min_{\theta} f(\theta|\bm\lambda_i),\ \forall i\in\{1,...,N\},
 \end{align*}
utilizing the inter-task similarity to speed up convergence rates across all. Hereinafter, we use notation $f_i(\theta) $ to depict $f(\theta|\bm\lambda_i)$ for simplicity.

To date, advances in multi-task optimization have mainly been in the domain of gradient-free evolutionary and Bayesian optimization \cite{swersky2013multi,bali2019multifactorial,gupta2022half}, where implicit transfer of useful knowledge is achieved through the process of evolution. By bringing the idea to gradient-based algorithms, \cite{bai2020multi} proposed the multi-task gradient descent algorithm to solve MTL. The method is further extended in \cite{bai2021multitask} to the multi-task evolution strategies for solving benchmark optimization problems as well as practical optimization problems. In this paper, we put forth a synergy of decomposition-based Pareto MTL with multi-task optimization, thus arriving faster at a representative subset of Pareto optimal models in only one algorithmic pass.

\section{Preliminaries}

\subsection{Casting Multi-Task Learning as Multi-Objective Optimization}
Consider a MTL problem over an input space $\mathcal X$ and a collection of task-specific output spaces $\{\mathcal Y^t\}_{t\in[m]}$, such that a large dataset of i.i.d. data points $\{x_k,y_k^1,...,y_k^m\}_{k\in [P]}$ is given, where $m$ is the number of tasks, $P$ is the number of data points, and $y_k^t$ is the label of the $t$-th task for the $k$-th data point. We consider a parametric hypothesis class per task as $g^t(x;\theta^{sh},\theta^t): \mathcal X\to\mathcal Y^t$ such that some parameters $\theta^{sh}$ are shared among all tasks and some parameters $\theta^t$ are task-specific. The task-specific loss function is then given by $L^t(\theta^{sh},\theta^t): \mathcal Y^t\times \mathcal Y^t\to\mathbb R$. Since some tasks could conflict with each other (i.e., improvements in one are accompanied by the worsening of another), it is possible formulate MTL in the form of MOO as follows:
\begin{align*}
\min_{[\theta^{sh},\theta^1,...,\theta^m] \in \mathbb R^d}  (L^1(\theta^{sh},\theta^1), L^2(\theta^{sh},\theta^2),...,L^m(\theta^{sh},\theta^m)).
\end{align*}
The goal is to find a collection of Pareto optimal solutions well distributed over the Pareto set, that can represent a variety of trade-offs among the tasks.

Note that since parameters $\theta^t$ have no influence on other tasks' loss functions, for simplicity, we drop the task-specific parameter terms in the MOO formulation and use $\theta$ instead of $\theta^{sh}$ to represent the shared parameters. The problem statement is then simplified as:
\begin{align}\label{mop}
\min_{\theta \in \mathbb R^d}  (L^1(\theta), L^2(\theta),...,L^m(\theta)).
\end{align}

\subsection{Recasting Multi-Objective Optimization as Multi-Task Optimization}
Decomposition is a basic strategy to tackle MOO problems. It aggregates different objective functions into a scalar-valued objective using a weight vector. The solution of the resulting single-objective optimization problem gives one Pareto optimal solution. Using different weight vectors, a set of different Pareto optimal solutions can thus be found. There are several approaches to transform a MOO problem into a number of scalar optimization subproblems. In what follows, we introduce two common approaches, the weighted-sum and the Tchebycheff scalarization.

Let $\bm{\lambda} = (\lambda^1,...,\lambda^m)^T$ be a weight vector that satisfies $\lambda^j\geq 0$ for $j=1,...,m$. 
The weighted-sum aggregates different objectives using a linear combination, such that the resulting scalar-valued objective function is:
\begin{align}\label{ws}
f^{ws}(\theta|\bm\lambda) = \sum_{j=1}^m\lambda^jL^j(\theta).
\end{align}
The weighted-sum approach works well for MOO with convex shaped Pareto fronts. However, it can not find solutions located in the non-convex parts of the front \cite{boyd2004convex}. In contrast, the Tchebycheff approach minimizes the following scalar-valued objective function:
\begin{align}\label{te}
f^{te}(\theta|\bm\lambda) = \max_{1\leq j\leq m}\{\lambda^j|L^j(\theta)-z^{j*}|\}, 
\end{align}
where $z^*=(z^{1*},...,z^{m*})^T$ is a reference point with each component being $z^{j*} = \min \{L^j(\theta)\}$. 
This decomposition is able to support solutions located in non-convex parts of Pareto fronts as well. However, the scalarized objective is non-smooth, which makes direct application of gradient-based optimization infeasible. To address this issue, we introduce smoothing techniques to obtain an approximate but differentiable scalarization. Precisely, we use the $\alpha_s$-softmax function $\frac{\sum_{i=1}^mx_ie^{\alpha_s x_i}}{\sum_{i=1}^me^{\alpha_s x_i}}$ to approximate $\max_{1\leq i\leq m}(x_i)$ \cite{lange2014applications}, and the commonly used $\sqrt{x^2+\epsilon}$ to approximate $|x|$, thus resulting in the following smoothed function: 
\begin{align}\label{ate}
\nonumber&f^{st}(\theta|\bm{\lambda}) = \\
&\frac{\sum_{j=1}^m \lambda^j\sqrt{(L^j(\theta)-z^{j*})^2+\epsilon}e^{\alpha_s\lambda^j\sqrt{(L^j(\theta)-z^{j*})^2+\epsilon}}}{\sum_{j=1}^m e^{\alpha_s\lambda^j\sqrt{(L^j(\theta)-z^{j*})^2+\epsilon}}},
\end{align}
where $\alpha_s>0$ and $\epsilon>0$ are smoothing parameters. The lager $\alpha_s$ and the smaller $\epsilon$ are, the smaller the approximation error.

To generate a set of Pareto optimal solutions, different weight vectors $\bm\lambda$ can be used to form different subproblems. Let $\bm\lambda_1,...,\bm\lambda_N$ be a set of evenly spread weight vectors in the objective space. Then, the problem of finding a good representation of the Pareto set can be recast (following Section 2.4) as one of multi-task optimization with $N$ tasks. 

Intuition suggests that neighboring subproblems (those generated by neighboring weight vectors; e.g., problems $i$ and $j$ are neighboring if $\|\bm\lambda_i-\bm\lambda_j\|\leq \epsilon$ with $\epsilon$ being a positive threshold) are likely to have similar optimal solutions. Note that this intuition is widely used with success in the design of decomposition-based evolutionary algorithms \cite{zhang2007moea}. Solving the subproblems jointly via multi-task optimization could then boost overall convergence rates by transferring useful information across similar tasks.

With this intuition, finding a set of Pareto optimal models for MTL can be transformed into the following multi-task optimization setup:
\begin{align}\label{subp}
\min_{\theta_i\in\mathbb R^d} f_i(\theta_i), \ \forall i= 1,...,N,
\end{align}
where $f_i(\theta_i)$ represents $f^{ws}(\theta|\bm\lambda_i)$ or $f^{st}(\theta|\bm\lambda_i)$. A schematic of the transformation is given in Fig.~\ref{ProbReformulation} (a). Each subproblem is a task in multi-task optimization, composed of all the loss functions (tasks in MTL) and the weight vectors. Fig.~\ref{ProbReformulation} (b) illustrates this for the case of MTL over a pair of tasks. Using four weight vectors, four subproblems are produced and solved jointly, with knowledge transfers occurring between nearest subproblems. As a result, four Pareto optimal solutions can be derived, with different trade-offs among the loss functions of the two tasks.

\definecolor{tab_orange}{RGB}{255, 127, 14}
\definecolor{tab_blue}{RGB}{31, 119, 180}
\definecolor{tab_cyan}{RGB}{23, 190, 207}
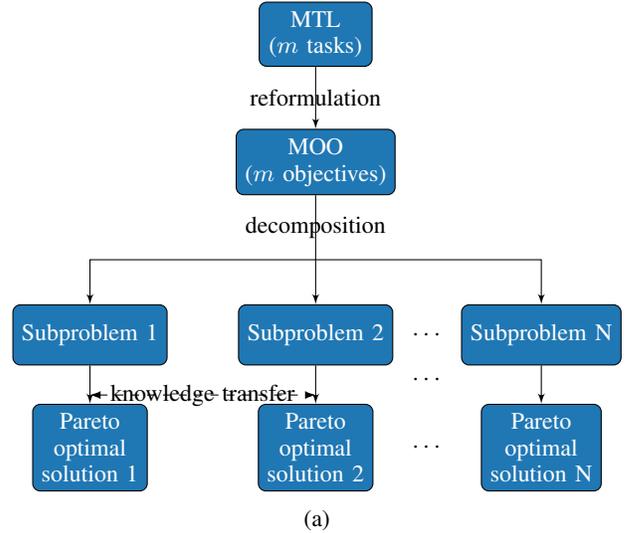
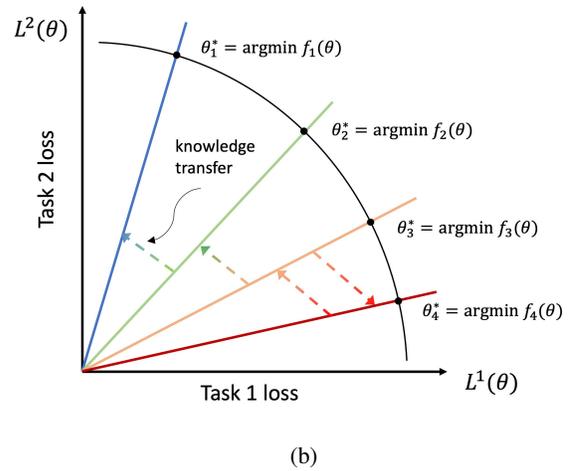
\begin{figure}[!htb]
\begin{subfigure}{\linewidth}
\centering
\begin{tikzpicture}[remember picture,
recnode/.style={rectangle, rounded corners=0.1cm, draw, fill=tab_blue, minimum width=15mm, minimum height = 8mm, text=white, align = center}
]
\tikzstyle{every node}=[font=\small]
\tikzstyle{line} = [draw, -latex']
    \node[recnode] (mtl) at (5,10) {MTL \\ ($m$ tasks)};
    \node[recnode] (moo) at (5,8.3) {MOO \\ ($m$ objectives)};
    \node[recnode] (s1) at (2,6) {Subproblem 1};
    \node[recnode] (s2) at (5,6) {Subproblem 2};
    \node (si) at (6.5,6) {$\cdots$};
    \node (transi) at (6.5,5.4) {$\cdots$};
    \node[recnode] (sN) at (8,6) {Subproblem N};
    \node[recnode] (pos1) at (2,4.5) {Pareto \\ optimal \\ solution 1};
    \node[recnode] (pos2) at (5,4.5) {Pareto \\ optimal \\ solution 2};
    \node (posi) at (6.5,4.5) {$\cdots$};
    \node[recnode] (posN) at (8,4.5) {Pareto \\ optimal \\ solution N};
    \path[line] (mtl) -- node[text width=2.5cm,midway,align=center]{reformulation} (moo);
    \draw (moo) -- node[text width=2.5cm,midway,align=center]{decomposition} (5,7);
    \path[line] (5,7) -- ++(-3,0) -- (s1);
    \path[line] (5,7) -- (s2);
    \path[line] (5,7) -- ++(3,0) --(sN);
    \path[line] (s1) -- (pos1);
    \path[line] (s2) -- (pos2);
    \path[line] (sN) -- (posN);
    \draw[latex-latex,dashed] (2,5.2) -- node[text width=2.5cm,midway,align=center]{knowledge transfer} (5,5.2);
\end{tikzpicture}
\caption{}
\end{subfigure} 

\vspace*{10pt}

\begin{subfigure}{\linewidth}
\centering
\includegraphics[width = 0.9\linewidth]{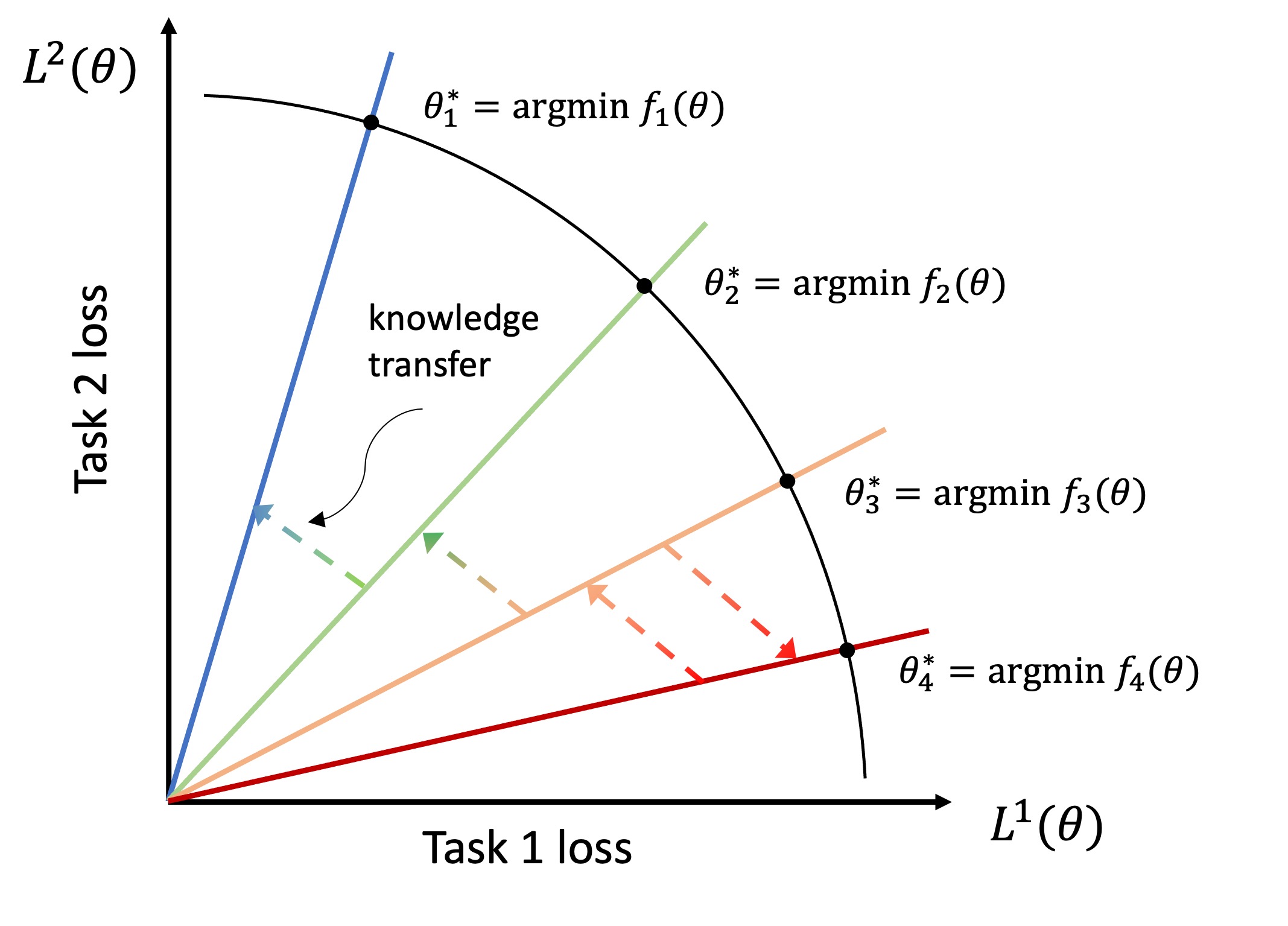}
\caption{}
\end{subfigure}
\caption{Finding a set of Pareto MTL models in one algorithmic pass by means of jointly solving  related subproblems with multi-task optimization. (a) Turning MTL into a set of subproblems. (b) Each subproblem provides one Pareto optimal model. Different Pareto optimal models embody different trade-offs among the tasks.}\label{ProbReformulation}
\end{figure}

\section{MT$^2$O and its Theoretical Analysis}

In this section, we present the Multi-Task Learning with Multi-Task Optimization (MT$^2$O) algorithm. As pointed out in Preliminaries, the MTL problem is first decomposed into $N$ scalar optimization subproblems representing different trade-offs among the tasks in the original MTL. Then, solving these subproblems simultaneously using a novel multi-task gradient descent method, a set of models embodying the different trade-offs can be obtained.

\subsection{Multi-Task Gradient Descent}

To simultaneously solve the $N$ optimization subproblems in \eqref{subp}, we propose a \emph{multi-task gradient descent} (MTGD) iteration to find the stationary points.

Under the assumption that there exist similarities among neighboring subproblems, the solution of one can facilitate the search for the solution of another. To harness these similarities, model parameters are iteratively transferred among the subproblems during the joint optimization process. The resulting MTGD iteration is:
\begin{align}\label{mgd}
\theta_i^{t+1} = \sum_{j=1}^N M_{ij}^t\theta_j^t-\alpha\nabla f_i(\theta_i^t),
\end{align}
where $t\in\mathbb N$ is the iteration index, $\nabla f_i(\theta_i^t)$ is the gradient of $f_i$ at $\theta_i^t$, and $M_{ij}^t$ is a $d \times d$ diagonal matrix representing the extent to which parameters transferred from subproblem $j$ inform the optimization of subproblem $i$. The transfer matrix is designed to satisfy the following conditions,
\begin{subequations}\label{m}
\begin{align}
&M_{ij,k}^t\geq0, \forall i, j= 1,...,N, k = 1,...,d, \label{m1}\\
&\sum_{j=1}^N M_{ij}^t=I_d,\label{m2}\\
&M_{ij,k}^t = 0,\ \text{for}\ t>T_0 \ \text{and}\  j\neq i,\label{m3}
\end{align}
\end{subequations}
where $M_{ij,k}^t$ is the $k$-th diagonal element in $M_{ij}^t$, which determines the transfer power of the $k$-th element in $\theta$. $T_0$ is a nonnegative integer. $I_d$ is the identity matrix with dimension $d\times d$. In the aforementioned conditions, \eqref{m1} implies that inter-task interactions are non-repulsive, \eqref{m2} imposes a sum-to-one normalization which limits the transfer power to prevent divergence, and \eqref{m3} implies that no transfer occurs after iteration $T_0$.

\subsection{Faster Convergence by Multi-task Transfer}

In this section, we theoretically analyze the convergence of the MTGD iteration for strongly convex and differentiable objective functions. Note that differentiability follows from the smoothening in (\ref{ate}). The objective functions of the subproblems $f_i: \mathbb R^d\to \mathbb R$ are also assumed to satisfy the following condition:
\begin{align*}
\xi_iI_d \leq \nabla^2 f_i(\theta_i)\leq L_{f_i}I_d,
\end{align*}
where $\xi_i$ and $L_{f_i}$ are positive constants satisfying $\xi_i\leq L_{f_i}$, and $\nabla^2f_i(\theta_i)$ is the Hessian of $f_i$ at $\theta_i$.

Under the conditions given in \eqref{m}, convergence of MTGD is assured as it falls back to pure gradient descent for $ t>T_0$. Further, suppose the transfer coefficients satisfy $M_{ij,k}^t = M_{ji,k}^t$, i.e., the transfer is symmetric, and $M_{ij}^t$ is chosen from a finite set. We show in \textbf{Theorem \ref{theo2}} below that under certain conditions, the induced transfer of parameter values among subproblems, as in (\ref{mgd}), accelerates convergence to a representative subset of Pareto MTL models. 

We first detail the terms and notations used in our derivation. Bold fonts depict concatenation. Specifically,
$\bm\theta = [\theta_1^{T}, ..., \theta_N^{T}]^T\in\mathbb R^{dN}$ where $\theta_i^T$ represents the transpose of $\theta_i$, $\theta_i^*$ denotes the optimal solution for subproblem $i$, $\bm\theta^* = [\theta_1^{*T}, ..., \theta_N^{*T}]^T\in\mathbb R^{dN}$, $\tilde {\bm\theta} = \bm\theta-\bm\theta^*$, and $\bm f(\bm\theta) = [ f_1(\theta_1),...,f_N(\theta_N)]^T\in\mathbb R^{N}$.
Let $\bar L_{f_i} = \max_i\{L_{f_i}\}$, $\underline {\xi_i} = \min_i\{\xi_i\}$. We also define $b_0 = \max_{i,j}\|\theta_i^*-\theta_j^*\|$, where $\|\cdot\|$ denotes the Euclidean norm,  $\eta_1=\max_t(\rho(\mathcal M^t-\alpha H^t))$ and  $\eta_2 = \max_t(\rho(I_{dT}-\alpha H^t))$. Here, $\rho(\cdot)$ represents the spectral radius of a matrix, $H^t = \int_0^1\nabla^2 \bm f(\bm\theta^*+\mu(\bm\theta^t-\bm\theta^*))d\mu\in\mathbb R^{dN\times dN}$ with $\nabla^2 \bm f(\bm\theta) = \text{diag}\{\nabla^2 f_1(\theta_1),...,\nabla^2 f_N(\theta_N)\}\in\mathbb R^{dN\times dN}$ where diag$\{\cdot\}$ is the operation that forms a block diagonal matrix with each element, and $\mathcal M^t$ is the matrix with its $i,j$-th block element being $M_{ij}^t$. Lastly, $H_i^t = \int_0^1\nabla^2 f_i(\theta_i^*+\mu(\theta_i^t-\theta_i^*))d\mu\in\mathbb R^{d\times d}$.

\begin{theorem}\label{theo2}
Suppose there exist $i$ and $j$ such that $H_i^t\neq H_j^t$ and the transfer coefficient $M^t_{ij}$ satisfies 
\begin{align*}
\nonumber &\sum_{j=1}^N M_{ij}^t = I_d,\\
&M_{ij,k}^t\geq 0,\ \forall i,j=1,...,N, k=1,...,d,\\
&M_{ij,k}^t = 0,\ \text{for}\ t>T_0 \ and\  j\neq i,\\
&M_{ii,k}^t\geq0.5,\\
&M_{ij,k}^t = M_{ji,k}^t
\end{align*}
where $T_0$ is a nonnegative integer satisfying
\begin{align}\label{InitialCond}
\eta_1^{T_0}\frac{\|\nabla \bm f(\bm \theta^0)\|}{\bar L_{f_i}}+\frac{1-\eta_1^{T_0}}{1+\eta_1}b_0< \eta_2^{T_0}\frac{\|\nabla \bm f(\bm\theta^0)\|}{\bar L_{f_i}},
\end{align}
then, under \eqref{mgd}, $\|\tilde {\bm\theta^t}\|$ converges to zero faster than when there is no transfer if $\exists\ T_0>0$ and the step size $\alpha$ satisfies
\begin{align}
0<\alpha<\frac{1}{2\bar L_{f_i}}.
\end{align}
\end{theorem}

\begin{proof}
Writing \eqref{mgd} into a concatenated form gives
\begin{align*}
\bm\theta^{t+1} = \mathcal{M}^t\bm\theta^t-\alpha\nabla \bm f(\bm\theta^t).
\end{align*}
Subtracting $\bm\theta^*$ from both sides of the above equation, we have
\begin{align*}
\tilde {\bm\theta}^{t+1} =& \mathcal{M}^t\tilde {\bm\theta}^t+(\mathcal{M}^t-I_{dN}){\bm\theta}^*-\alpha(\nabla \bm f({\bm\theta}^t)-\nabla \bm f({\bm\theta}^*))\\
=& (\mathcal{M}^t-\alpha H^t)\tilde{\bm\theta}^t+(\mathcal{M}^t-I_{dN}){\bm\theta}^*\\
=& A_m^t\tilde {\bm\theta}^t+(\mathcal{M}^t-I_{dN}){\bm\theta}^*,
\end{align*}
where $A_m^t = \mathcal{M}^t-\alpha H^t$. Due to the assumption that $\mathcal{M}^t$ is symmetric, $A_m^t$ is symmetric. Thus,
\begin{align}\label{the2ite}
\nonumber\|\tilde {\bm\theta}^{t+1}\|\leq &\|A_m^t\tilde {\bm\theta}^t\|+\|(\mathcal{M}^t-I_{dN}){\bm\theta}^*\|\\
\nonumber\leq &\rho(A_m^t)\|\tilde {\bm\theta}^t\|+\|(\mathcal{M}^t-I_{dN}){\bm\theta}^*\|\\
\nonumber\leq & \prod_{\tau=0}^t\rho(A_m^{\tau})\|\tilde {\bm\theta}^0\|+\|(\mathcal{M}^t-I_{dN}){\bm\theta}^*\|\\
&+\sum_{\tau=0}^{t-1}\prod_{r=\tau+1}^{t}\rho(A_m^r)\|(\mathcal{M}^{\tau}-I_{dN}){\bm\theta}^*\|.
\end{align}

Denote ${\bm\theta}_s$ the variable when there is no transfer. Then, the iteration of ${\bm\theta}_s$ is
\begin{align*}
\tilde {\bm\theta}_s^{t+1} =& {\bm\theta}_s^t-{\bm\theta}^*-\alpha\nabla \bm f({\bm\theta}_s^t)\\
=&(I_{dN}-\alpha H^t)\tilde {\bm\theta}_s^t. 
\end{align*}
It can be derived that
\begin{align}\label{ites}
\|\tilde {\bm\theta}_s^{t+1}\|\leq \prod_{\tau=0}^t\rho(A_s^{\tau})\|\tilde {\bm\theta}^0\|,
\end{align}
where $A_s^t=I_{dN}-\alpha H^t$.

Let $\Delta^t = I_{dN}-\mathcal{M}^t$, it is obvious that
\begin{align}\label{rela_matrix}
A_s^t = A_m^t+\Delta^t.
\end{align}
Let $\lambda_m(B)$ be the decreasing ordered eigenvalues of matrix $B\in\mathbb R^{dN\times dN}$, i.e., $\lambda_1(B)\geq\lambda_{2}(B)\geq\cdots\geq\lambda_{dN}(B)$. Let $\bm r_i$ be the right eigenvector of $\mathcal{M}^t$ corresponding to eigenvalue $\lambda_i(\mathcal{M}^t)$, i.e., $\mathcal{M}^t \bm r_i = \lambda_i(\mathcal{M}^t)\bm r_i$. Then, $\bm r_i$ is also a right eigenvector of $\Delta^t$ corresponding to eigenvalue $1-\lambda_i(\mathcal{M}^t)$, which can be obtained by the following relationship,
\begin{align*}
\Delta^t\bm r_i = (I_{dN}-\mathcal{M}^t)\bm r_i = \bm r_i-\lambda_i(\mathcal{M}^t)\bm r_i = (1-\lambda_i(\mathcal{M}^t))\bm r_i.
\end{align*}
Thus, the eigenvalues of matrix $\Delta^t$ is $\lambda_1(\Delta^t) = 1-\lambda_{dN}(\mathcal{M}^t)$ and $\lambda_{dN}(\Delta^t) = 1-\lambda_1(\mathcal{M}^t)=0$. As a result, $\Delta^t$ is positive semidefinite. Since $A_s^t, A_m^t, \Delta^t$ are all Hermitian matrix and $\Delta^t\geq 0$, from Weyl's theorem \cite{horn1990matrix}, we have
\begin{align*}
\lambda_i(A_m^t)\leq \lambda_i(A_s^t), i=1,...,dN.
\end{align*}
The equality holds if and only if there is a nonzero vector $\bm x$ such that $A_m^t\bm x = \lambda_i(A_m^t)\bm x$, $\Delta^t\bm x=0$, and $A_s^t\bm x = \lambda_i(A_s^t)\bm x$. For $\Delta^t\bm x = 0$, we have $\bm x = \mathcal{M}^t\bm x$, which indicates that $\bm x$ is in the space spanned by the columns of $\bm 1_N\otimes I_d$ for $M^t\neq I_N$, where $\bm 1_N$ represents a $N$ dimensional column vector with each component being 1 and $\otimes$ is the Kronecker product. $A_s^t\bm x = \lambda_1(A_s^t)\bm x$ indicates $\bm x-\alpha H^t\bm x = \lambda_1(A_s^t)\bm x$, which gives $\alpha H^t\bm x = (1-\lambda_1(A_s^t))\bm x$. Since there exist $i,j$ such that $H_i^t\neq H_j^t$, such $\bm x$ does not exist. As a result, $\lambda_1(A_m^t)<\lambda_1(A_s^t)$. 

From the relation $A_m^t = \mathcal{M}^t-\alpha H^t$ and the Weyl's theorem, we have
\begin{align*}
\lambda_i(\mathcal{M}^t)+\lambda_{dN}(-\alpha H^t)\leq\lambda_i(A_m^t)\leq \lambda_i(\mathcal{M}^t)+\lambda_1(-\alpha H^t).
\end{align*}
Since $\xi_iI_d\leq\nabla^2 f_i(\theta_i)\leq L_{f_i}I_d$, $\underline{\xi_i}\leq\lambda_i(H^t)\leq \bar L_{f_i}$. From the stochastic property of $\mathcal{M}^t$, $\lambda_1(\mathcal{M}^t)=1$ and $\lambda_{dN}(\mathcal{M}^t)\geq-1$. Thus
\begin{align*}
1-\alpha \bar L_{f_i}\leq&\lambda_1(A_m^t)\leq 1-\alpha\underline{\xi_i},\\
\lambda_{dN}(\mathcal{M}^t)-\alpha \bar L_{f_i}\leq&\lambda_{dN}(A_m^t)\leq \lambda_{dN}(\mathcal{M}^t)-\alpha\underline{\xi_i}.
\end{align*}
Under the following conditions
\begin{align*}
\alpha < \frac{1}{2\bar L_{f_i}},\\
M_{ii,k}^t\geq0.5,
\end{align*}
we have $2\alpha \bar L_{f_i}-1<0$ and $\lambda_{dN}(\mathcal{M}^t)\geq 0$. Thus,
\begin{align*}
|\lambda_N(M^t)-\alpha \bar L_{f_i}|<1-\alpha \bar L_{f_i}.
\end{align*}
As a result, $\rho(A_m^t) = \max\{|\lambda_1(A_m^t)|,|\lambda_{dN}(A_m^t)|\} = \lambda_1(A_m^t)$. Together with the result that $\lambda_1(A_m^t)<\lambda_1(A_s^t)$, we can conclude that $\rho(A_m^t)<\rho (A_s^t)$ for $\mathcal{M}^t\neq I_{dN}$.

Since $\eta_1=\max_t(\rho(A_m^t))$ and $\eta_2 = \max_t(\rho(A_s^t))$, we have $\eta_1<\eta_2$. From iteration \eqref{the2ite}, we have
\begin{align*}
\|\tilde {\bm\theta}^{T_0}\|\leq\eta_1^{T_0}\|\tilde {\bm\theta}^0\|+\frac{1-\eta_1^{T_0}}{1+\eta_1}b_0.
\end{align*}
From \eqref{ites}, we have
\begin{align*}
\|\tilde {\bm\theta}_s^{T_0}\|\leq \eta_2^{T_0}\|\tilde {\bm\theta}^0\|.
\end{align*}

From the Lipschitz continuous assumption on $\nabla f_i$, we have $\|\tilde {\bm\theta}\|\geq \|\nabla \bm f({\bm\theta})\|/\bar L_{f_i}$. Thus
\begin{align*}
\eta_1^{T_0}+\frac{(1-\eta_1^{T_0})b_0}{(1+\eta_1)\|\tilde {\bm \theta}^0\|}\leq \eta_1^{T_0}+\frac{(1-\eta_1^{T_0})b_0\bar L_{f_i}}{(1+\eta_1)\|\nabla \bm f({\bm\theta}^0)\|}.
\end{align*}
By setting $T_0$ to satisfy
\begin{align*}
\eta_1^{T_0}+\frac{(1-\eta_1^{T_0})b_0\bar L_{f_i}}{(1+\eta_1)\|\nabla \bm f({\bm\theta}^0)\|}\leq \eta_2^{T_0},
\end{align*}
we have
\begin{align*}
\eta_1^{T_0}\|\tilde {\bm\theta}^0\|+\frac{1-\eta_1^{T_0}}{1+\eta_1}b_0< \eta_2^{T_0}\|\tilde {\bm\theta}^0\|,
\end{align*}
indicating that $\|\tilde {\bm\theta}^{T_0}\|$ possesses a tighter upper bound compared to $\|\tilde {\bm\theta}_s^{T_0}\|$. After $T_0$, MTGD falls back to pure gradient descent. As a result, MTGD converges faster to the optimal solution than the single task gradient descent.
\end{proof}

\subsection{Summary of the Proposed MT$^2$O}

A pseudocode of the proposed algorithm is shown in Algorithm~\ref{alg}. The output of MT$^2$O is a representative subset of Pareto optimal models by solving \eqref{mop} via decomposition. Step 9 is our core contribution, where model parameters are updated in a gradient descent direction whilst also utilizing transferred parameter values from neighboring subproblems. Here, $M^t_{ij}\in\mathbb R^{d\times d}$ is a diagonal matrix mandating the extent of transfer from subproblem $j$ to subproblem $i$. The $k$-th diagonal element in $M_{ij}^t$ is the transfer coefficient for the $k$-th element of the variables.
\begin{algorithm}[tb]
   \caption{Pseudocode of the Proposed MT$^2$O}
   \label{alg}
\begin{algorithmic}[1]
   \STATE {\bfseries Input:} 
\begin{enumerate}
  \item MTL problem \eqref{mop}
  \item $N$: number of the subproblems
  \item $\bm\lambda_1$,...,$\bm\lambda_N$: a unified spread of $N$ weight vectors
  \item $\alpha$: learning rate
\end{enumerate}
   \STATE {\bfseries Output:} A set of solutions $\theta_i^*, i\in\{1,...,N\}$:\\
$$\theta_i^* = \text{argmin}_{\theta} f_i(\theta)$$
   \STATE Generate an initial $\theta_i^0$ for $i\in\{1,...,N\}$.
   \STATE Set $t=0$
   \WHILE{stopping criterion is not met}
   \STATE Calculate the transfer matrix $M^t_{ij}$ for $i,j\in\{1,...,N\}$.
   \FOR {i = 1,...,N}
   \STATE Calculate gradients $\nabla f_i(\theta_i)$
   \STATE $\theta_i^{t+1} = \sum_{j=1}^NM_{ij}^t\theta_i^t-\alpha\nabla f_i(\theta_i)$
   \ENDFOR
   \STATE $t = t+1$
   \ENDWHILE
\end{algorithmic}
\end{algorithm}

With the intuition that the smaller the distance between a pair of weight vectors the more strongly correlated the corresponding subproblems, a scalar transfer coefficient between subproblems $i$ and $j$ can be defined based on the Euclidean distance between $\bm\lambda_i$ and $\bm\lambda_j$. Note that more sophisticated transfer coefficients could also be defined, such as element-wise and adaptive transfer coefficients, as long as the conditions in Theorem~1 are satisfied.


\subsection{Time Complexity}
When using predefined transfer coefficients, the time complexity of Algorithm~1 mainly comes from the updates of the model parameters. Since each $f_i$ is composed of $m$ task losses, the calculation of gradients in one subproblem leads to a complexity of $O(dm)$. The transfer leads to a complexity of $O(dN)$. Therefore, the overall complexity of the proposed MT$^2$O in one iteration and one subproblem is of order $O(d(m+N))$, which scales linearly with the parameter dimension $d$, number of tasks $m$, and the number of candidate solutions $N$.

\section{Experiments}
In this section, we first conduct experiments on synthetic examples to demonstrate the effectiveness and efficiency of our MT$^2$O algorithm. Then, we perform real-world MTL experiments on various types of learning tasks with different task numbers, including image classification, data regression, and hybrid classification and estimation.

We compare MT$^2$O with the following classical and state-of-the-art MTL algorithms: 
{\bf 1) Linear Scalarization (LS) \cite{gunantara2018review}}: minimize the weighted-sum of different tasks; {\bf 2) Pareto MTL (PMTL)} \cite{lin2019pareto}: find a set of Pareto solutions using multiple gradient descent with constraints; {\bf 3) Exact Pareto Optimal (EPO)} \cite{mahapatra2020multi}: find a set of Pareto solutions with user preferences; {\bf 4) MTL as Multi-Objective Optimization (MTLMOO)} \cite{sener2018multi}: find one arbitrary Pareto solution using multiple gradient descent; {\bf 5) Pareto HyperNetwork (PHN)} \cite{navon2021learning}: train a hypernetwork to learn the Pareto front; {\bf 6) MOO using Stein variational gradient descent (MOOSVGD)} \cite{liu2021profiling}: update a set of points towards the Pareto front while encouraging diversity among particles; and {\bf 7) Conditioned one-shot multi-objective search (COSMOS)} \cite{ruchte2021scalable}: contextualize network preferences into the input space to achieve well-distributed Pareto fronts. We run the experiments based on open-sourced codes for comparing algorithms PMTL, MTLMOO, EPO \footnote{https://github.com/dbmptr/EPOSearch}, PHN \footnote{https://github.com/AvivNavon/pareto-hypernetworks}, MOOSVGD \footnote{https://github.com/gnobitab/MultiObjectiveSampling}, and COSMOS \footnote{https://github.com/ruchtem/cosmos}. Note that when using weighted-sum as the decomposition method and there is no transfer among the subproblems, MT$^2$O is identical to LS.


For simplicity, we refer to both the preference vector and weight vector as the reference vector in the following discussion. In our proposed method, MT$^2$O, we assess subproblem similarity based on the Euclidean distances between reference vectors, except for the CelebA dataset experiment, which employs cosine similarities of the labels (further details to follow). Let $J$ represent the neighborhood size. The calculation of transfer coefficients $M_{ij,k}, j\in{1,...,N}, k\in{1,...,d}$ proceeds as follows: First, we compute the Euclidean distances between all reference vectors. Next, for each reference vector, we identify the closest $J$ reference vectors, including the reference vector itself. Subsequently, we set the transfer coefficient between subproblem $i$ and its nearest neighbor as $\frac{J}{2(1+...+J)}+\frac{1}{2}$, the second nearest as $\frac{J-1}{2(1+...+J)}$, the third nearest as $\frac{J-2}{2(1+...+J)}$, and so forth. Finally, transfer coefficients for the remaining $N-J$ subproblems are assigned zero. The hyperparameters, including $T_0$ and $J$, are determined empirically. We set $T_0$ based on the insights from Theorem~1, where smaller values are required for variables closer to optimal solutions. Consequently, we opt for relatively small $T_0$ values: 10 in synthetic examples and 30 in multi-task learning tasks. The choice for $J$ aligns with the number of tasks since more tasks entail a greater number of subproblems, hence necessitating a larger neighborhood size. Additionally, other hyperparameters adhere to established methodologies. Specifically, in learning tasks, we employ the SGD with a learning rate of $\alpha = 1e-3$ in methods except PHN, where default parameters are utilized. We acknowledge that further hyperparameter tuning, like grid search, could enhance performance. However, given the primary focus of this paper on multi-task optimization, we omitted this step in the current version.

In addition to training losses and testing accuracies, we also use Hypervolume (HV) to evaluate the obtained Pareto solutions. HV is strictly monotonic with regard to Pareto dominance and can measure the proximity to the Pareto front and diversity simultaneously \cite{zitzler1999multiobjective}. For a given set of points $S\in\mathbb R^d$ and a reference point $r\in\mathbb R^d_+$, the HV of $S$ is measured by the region of non-dominated points in $S$ bounded above by $r$. The larger the HV value is, the better the solution is.

\subsection{Synthetic Examples}

To better understand the behaviors of the algorithms, we first test the algorithms on three synthetic MOO problems. Problem one (P1) is from \cite{lin2019pareto}, which has a concave Pareto front, and the other two problems are ZDT1 and ZDT2 \cite{zitzler2000comparison}, which are commonly used MOO benchmarks. In all problems, a two-objective optimization problem $\min_{\theta\in\mathbb R^d} {\bf L}(\theta) = (L^1(\theta), L^2(\theta))$ is solved, where $d$ is set to 20. 
The objective functions in P1, ZDT1, and ZDT2 are listed below:

P1:
\begin{align*}
L^1(\theta) = 1-e^{-\|\theta-\frac{1}{\sqrt{d}}\|^2},\ L^2(\theta) = 1-e^{-\|\theta+\frac{1}{\sqrt{d}}\|^2}
\end{align*}

ZDT1:
\begin{align*}
L^1(\theta) = \theta_1, \ L^2(\theta) = g(\theta)[1-\sqrt{\theta_1/g(\theta)}]
\end{align*}

ZDT2:
\begin{align*}
L^1(\theta) = \theta_1, \ L^2(\theta) = g(\theta)[1-(\theta_1/g(\theta))^2]
\end{align*}
with $g(\theta) = 1+\frac{9}{d-1}\sum_{i=2}^d\theta_i$ and $\theta_i\in[0,1]$ for ZDT1 and ZDT2.

We run all the algorithms with 10 evenly distributed reference vectors. The initial values for the subproblems are evenly distributed within the design space. The step size is set to 1 for P1 and 0.3 for ZDT1 and ZDT2, which are found empirically to be near optimum across all architectures. Since the Pareto fronts of P1 and ZDT2 are concave, we use the smoothed Tchebycheff approach to decompose the MOO into subproblems in MT$^2$O. The smooth parameters are set as $\alpha_s = 5$ and $\epsilon = 0.05$ by experience, and the reference point $z^*$ can be calculated directly from the analytical objective functions. In PHN and COSMOS, based on \cite{navon2021learning}, a two-layer multilayer perceptron (MLP) with 50 hidden units on each layer is used as the Hypernetwork or incorporate the reference vectors into the input space, and the parameter of Dirichlet distribution is set to 0.2. PHN and COSMOS are run for 500 iterations, and other methods are run for 50 iterations for each reference vector. To compare the performance of the algorithms, we scale the iterations of PHN and COSMOS to 50 and plot the HV value convergence in one figure. Thirty independent runs are conducted to produce reliable performance statistics.

\begin{figure*}[!htb]
\centering
 \begin{subfigure}[b]{0.32\linewidth}
    \includegraphics[width=\textwidth]{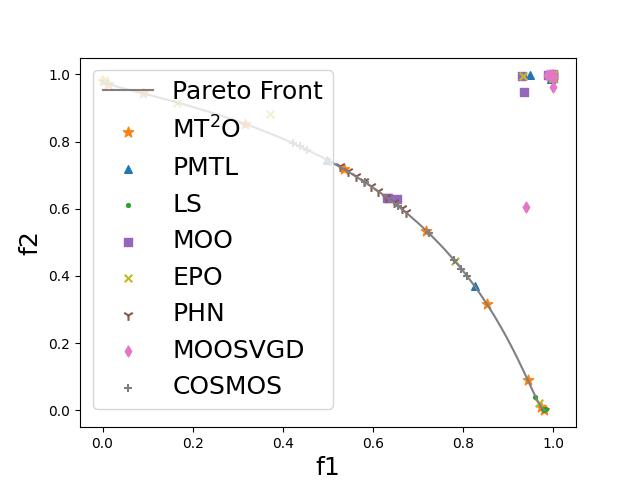}
  \end{subfigure} 
  \begin{subfigure}[b]{0.32\linewidth}
    \includegraphics[width=\textwidth]{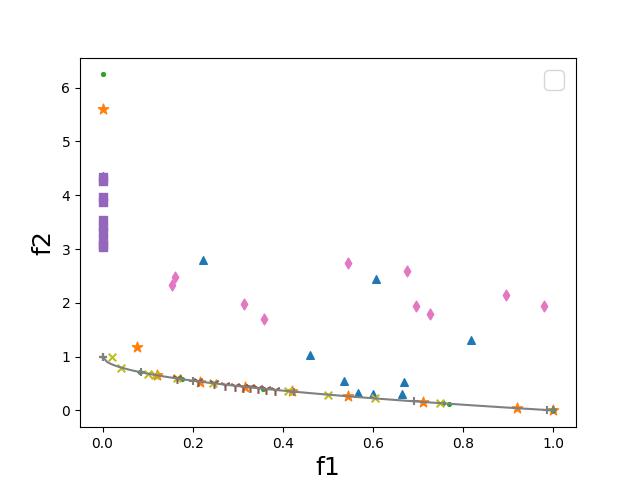}
  \end{subfigure} 
   \begin{subfigure}[b]{0.32\linewidth}
    \includegraphics[width=\textwidth]{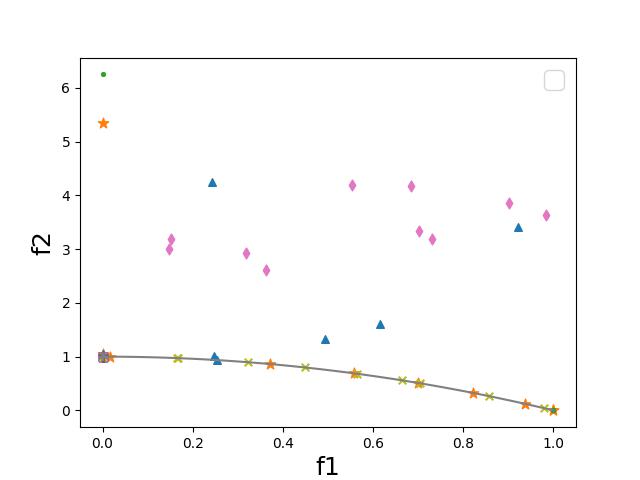}
  \end{subfigure} 
  \begin{subfigure}[b]{0.32\linewidth}
    \includegraphics[width=\textwidth]{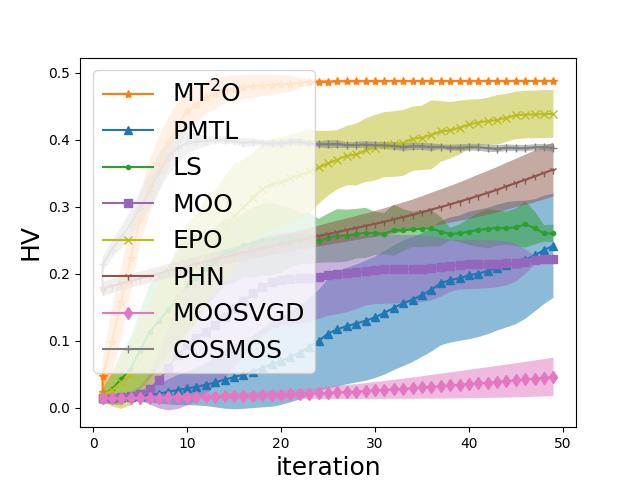}
    \caption{P1}
  \end{subfigure} 
  \begin{subfigure}[b]{0.32\linewidth}
    \includegraphics[width=\textwidth]{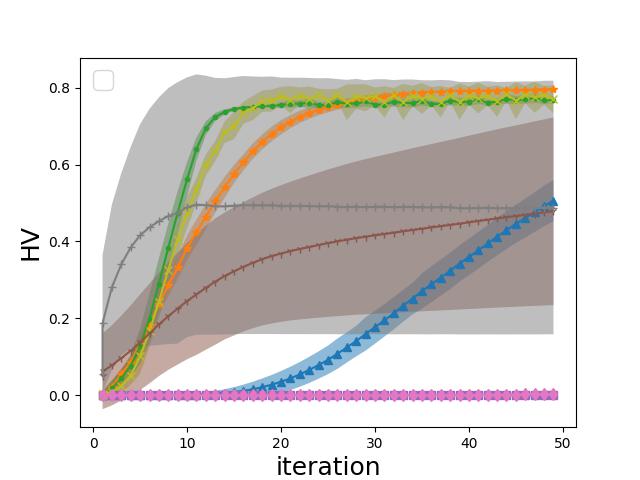}
    \caption{ZDT1}
  \end{subfigure} 
   \begin{subfigure}[b]{0.32\linewidth}
    \includegraphics[width=\textwidth]{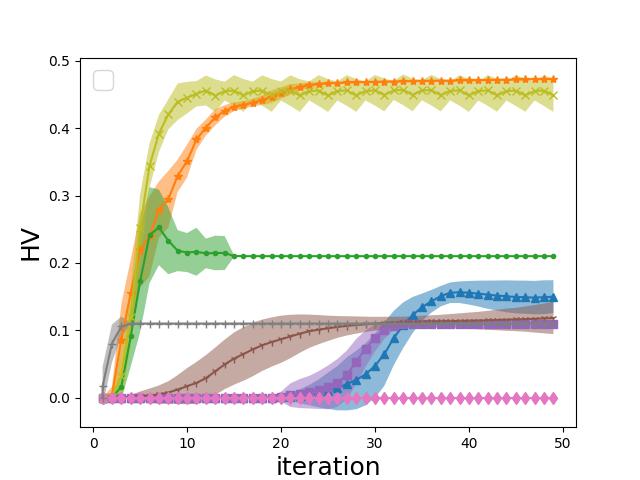}
    \caption{ZDT2}
  \end{subfigure} 
  \caption{The results for the synthetic examples. The top row shows the approximated Pareto front of the first run, and the bottom row shows the HV value convergence curves during optimization, calculated using the reference point (1.1,1.1). The HV values are averaged over 30 runs.}\label{synthetic}
\end{figure*}

The results obtained by different algorithms are shown in Fig.~\ref{synthetic}. The top row shows the obtained Pareto solutions of the first run, and the bottom row shows the averaged HV value during the optimization process. As can be seen, the proposed MT$^2$O using approximated Tchebycheff decomposition can successfully find a set of well-distributed Pareto solutions with different trade-offs. Notably, in P1 and ZDT2, since the Pareto fronts are concave, we can observe that only extreme solutions can be found by LS. MTLMOO can only find
solutions having similar trade-offs in all problems since it does not consider the reference vector information. PMTL, EPO, and COSMOS have large standard deviations as shown in Fig.~\ref{synthetic}~(a) implying that the methods are sensitive to the initial values of the variables. Since PHN and COSMOS train a single hypernetwork for all reference vectors, we optimized it for a higher number of iterations than the other methods. However, without substantial effort to fine-tune the hyperparameters, it was found to be very difficult to achieve satisfying solutions.

To show that transfer among subproblems indeed accelerates convergence, we compare the HV value obtained for the cases of with and without transfer in MT$^2$O, while using Tchebycheff decomposition scheme. Table~\ref{hv_trans} shows the averaged HV value over 30 runs at iterations 1, 10, 30, and 50. The results show that the use of transfer can bring the benefit of accelerating search convergences, especially at the early stages. As a result, the solutions obtained within a small number of iterations are found to be of higher quality due to the transfer. Conducting the Wilcoxon rank sum test with a 95\% confidence level, the obtained final HV value with transfer is significantly better than that without transfer.

\begin{table}[!htb]
    \centering
    \normalsize\caption{The HV value obtained at different iterations of MT$^2$O with and without transfer among subproblems. The better results are highlighted in bold.}\label{hv_trans}  
    \normalsize\centering\begin{tabular*}{\hsize}{@{}@{\extracolsep{\fill}{}}l|c|cccc@{}}
        \hline\hline
         \multirow{2}{*}{Problem} & \multirow{2}{*}{Transfer} & \multicolumn{4}{c}{Iteration} \\
         && 1 & 10 & 30 & 50  \\
         \hline
         \multirow{2}{*}{P1} & With Trans & \textbf{0.0347} & \textbf{0.4089} & \textbf{0.4872} & \textbf{0.4875} \\
         & No Trans & 0.0236 & 0.2884 & 0.3688 & 0.4675 \\
         \hline
         \multirow{2}{*}{ZDT1} & With Trans & 0 & \textbf{0.3098} & \textbf{0.7714} & \textbf{0.7946} \\
         & No Trans & 0 & 0.2720 & 0.7687 & 0.7922 \\
         \hline
         \multirow{2}{*}{ZDT2} & With Trans & 0 & \textbf{0.3141} & \textbf{0.4692} & \textbf{0.4741} \\
         & No Trans & 0 & 0.2888 & 0.4676 & 0.4714 \\
        \hline\hline
    \end{tabular*}
\end{table}

\subsection{Multi-Task Learning Datasets}

\definecolor{tab_orange}{RGB}{255, 127, 14}
\definecolor{tab_blue}{RGB}{31, 119, 180}
\definecolor{tab_cyan}{RGB}{23, 190, 207}
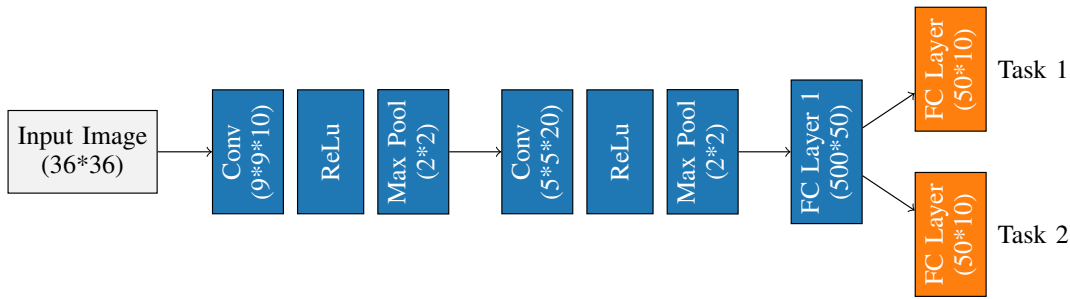
\begin{figure*}[h]
\centering
\resizebox{0.8\linewidth}{!}{
\begin{tikzpicture}[
recnode1/.style={rectangle, draw, fill=tab_blue, minimum width=15mm, minimum height = 8mm,rotate=90, text=white, align = center},
recnode2/.style={rectangle, draw, fill=tab_orange, minimum width=15mm, minimum height = 8mm,rotate=90, text = white, align=center},
]
\tikzstyle{every node}=[font=\small]
    \node[rectangle, draw, fill = gray!10, minimum size = 10mm, align=center] (image) at (-0.5,0) {Input Image\\ (36*36)};
    \node[recnode1] (conv1) at (1.5,0) {Conv \\ (9*9*10)};
    \node[recnode1] (relu1) at (2.5,0) {ReLu};
    \node[recnode1] (pool1) at (3.5,0) {Max Pool \\ (2*2)};
    \node[recnode1] (conv2) at (5,0) {Conv \\ (5*5*20)};
    \node[recnode1] (relu2) at (6,0) {ReLu};
    \node[recnode1] (pool2) at (7,0) {Max Pool \\ (2*2)};
    \node[recnode1] (f3) at (8.5,0) {FC Layer 1 \\ (500*50)};
    \node[recnode2] (lf4) at (10,1) {FC Layer \\ (50*10)};
    \node[recnode2] (rf4) at (10,-1) {FC Layer \\ (50*10)};
    \node[] (t1) at (11,1) {Task 1};
    \node[] (t2) at (11,-1) {Task 2};

    \path (image) edge[->] (conv1);
    \path (pool1) edge[->] (conv2);
    \path (pool2) edge[->] (f3);
    \path (f3) edge[->] (lf4);
    \path (f3) edge[->] (rf4);
\end{tikzpicture}
}
\caption{Architecture of the MTL network used for each subproblem for the  MultiMNIST, MultiFashionMNIST, Multi-(Fashion+MNIST) datasets.}\label{lenet}
\end{figure*}

\subsubsection{Image Classification}

We use three multi-task learning benchmark datasets, MultiMNIST \cite{sabour2017dynamic}, MultiFashionMNIST \cite{lin2019pareto}, and Multi-(Fashion+MNIST) \cite{lin2019pareto}. MNIST \cite{lecun1998gradient} is a famous public dataset of handwritten digits, which has a training set of 60,000 samples and a test set of 10,000 samples. FashionMNIST \cite{xiao2017fashion} is a public dataset of clothing images, associated with a label from 10 classes. It also has a training set of 60,000 samples and a test set of 10,000 samples. The MultiMNIST dataset is generated by combining two randomly picked images from the MNIST dataset to form one new image, which has one digit on the top-left and the other on the bottom-right. The MultiFashionMNIST is generated in a similar way to the FashionMNIST dataset. In Multi-(Fashion+MNIST) dataset, the top-left image is from MNIST and the bottom-right image is from FashionMNIST. For each dataset, there are two tasks, classifying the top-left image and classifying the bottom-right image. There are 120,000 samples in the training dataset and 20,000 samples in the test dataset. We downloaded the same datasets as in \cite{lin2019pareto} \footnote{Downloaded from: https://github.com/Xi-L/ParetoMTL}. 

We adopt the same base MTL neural network structure used in \cite{lin2019pareto}, derived from LeNet \cite{lecun1998gradient}. The structure of the MTL network used for each subproblem is shown in Fig.~\ref{lenet}. Illustrated in Fig.~\ref{lenet}, our base MTL network involves shared parameters for convolutional layers and the first fully connected layer, while task-specific parameters belong to the last fully connected layer. Cross entropy losses are used for training. We use the weighted-sum decomposition approach in MT$^2$O. The baseline is obtained by training individual tasks separately. 
We conduct training over 100 epochs with a batch size of 256. Five evenly distributed reference vectors are used.

\begin{figure*}[!htb]
\centering
 \begin{subfigure}[b]{0.32\linewidth}
    \includegraphics[width=\textwidth]{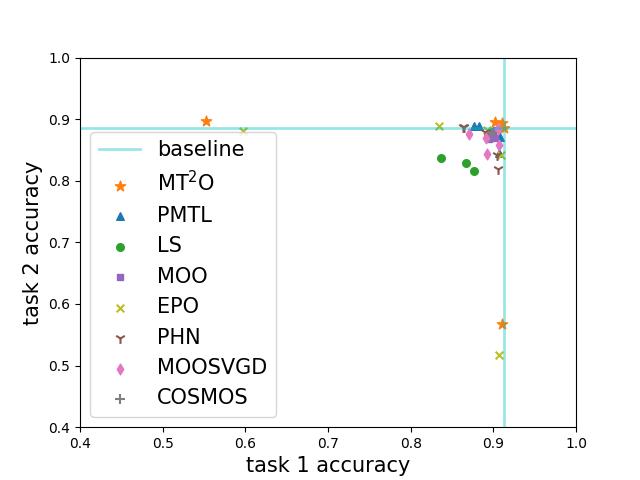}
  \end{subfigure} 
  \begin{subfigure}[b]{0.32\linewidth}
    \includegraphics[width=\textwidth]{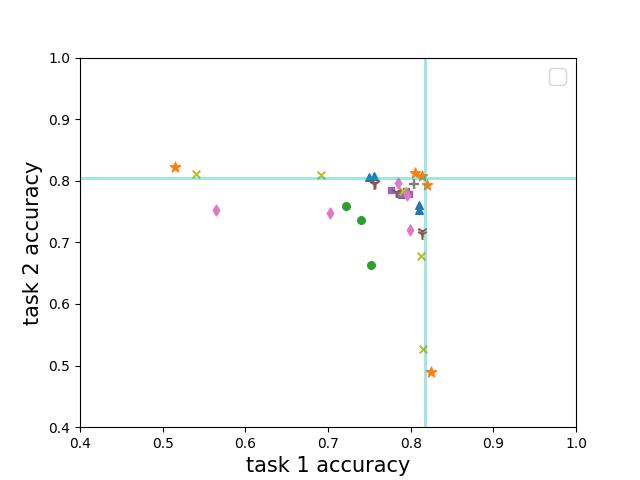}
  \end{subfigure} 
   \begin{subfigure}[b]{0.32\linewidth}
    \includegraphics[width=\textwidth]{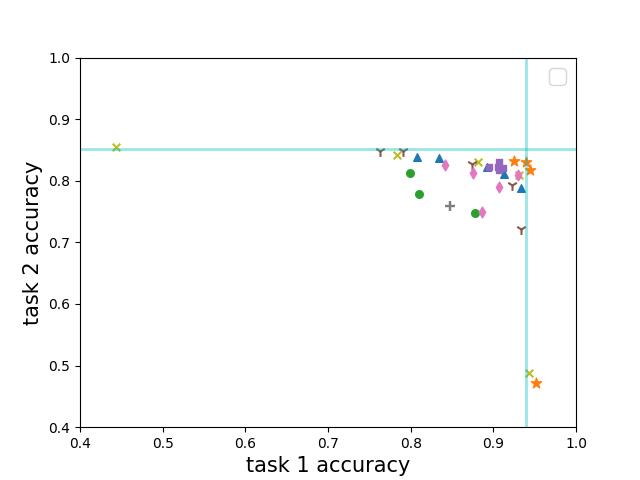}
  \end{subfigure} 
  \begin{subfigure}[b]{0.32\linewidth}
    \includegraphics[width=\textwidth]{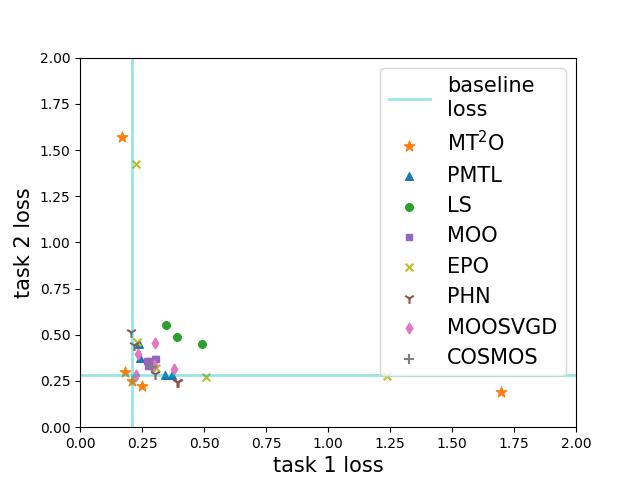}
  \end{subfigure} 
  \begin{subfigure}[b]{0.32\linewidth}
    \includegraphics[width=\textwidth]{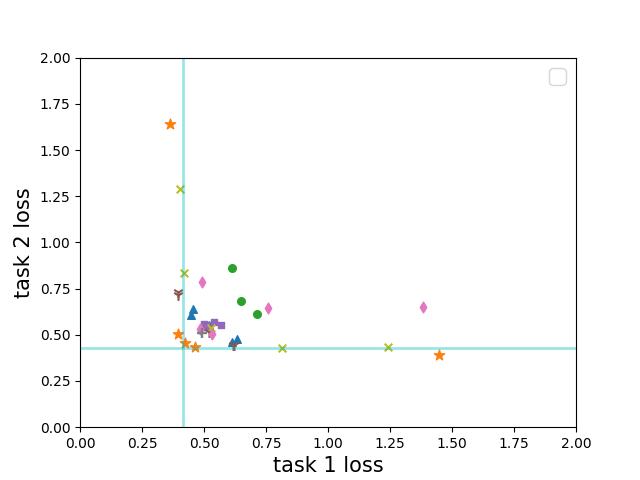}
  \end{subfigure} 
   \begin{subfigure}[b]{0.32\linewidth}
    \includegraphics[width=\textwidth]{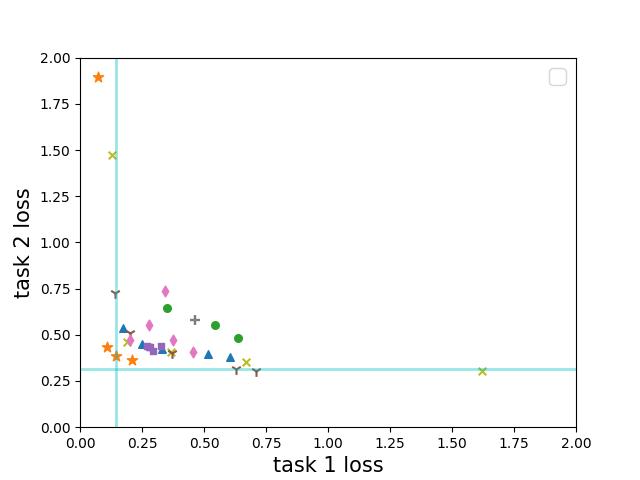}
  \end{subfigure} 
  \begin{subfigure}[b]{0.32\linewidth}
    \includegraphics[width=\textwidth]{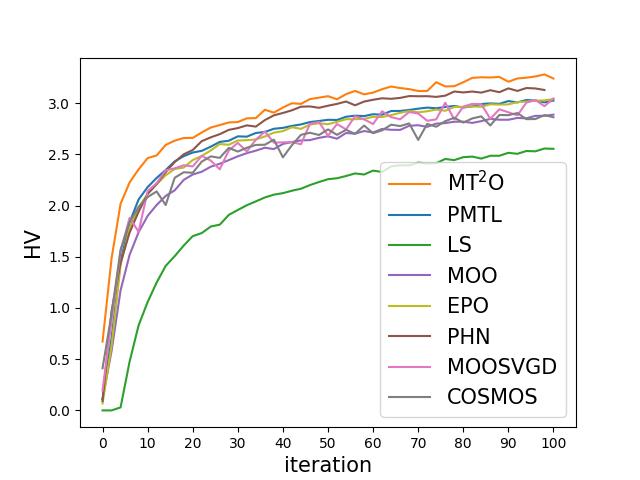}
    \caption{MultiMNIST}
  \end{subfigure} 
  \begin{subfigure}[b]{0.32\linewidth}
    \includegraphics[width=\textwidth]{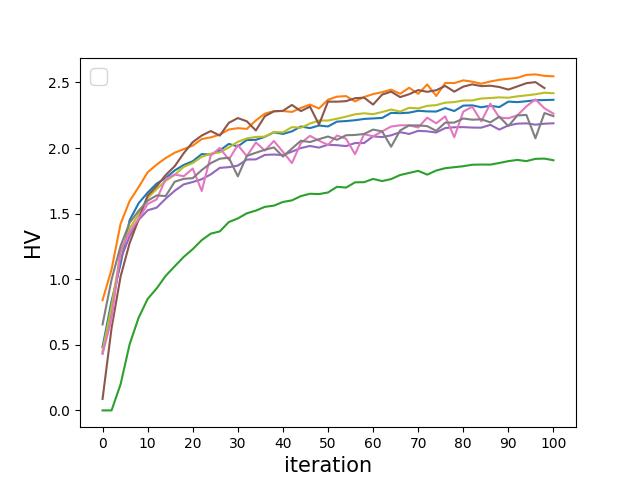}
    \caption{MultiFashionMNIST}
  \end{subfigure} 
   \begin{subfigure}[b]{0.32\linewidth}
    \includegraphics[width=\textwidth]{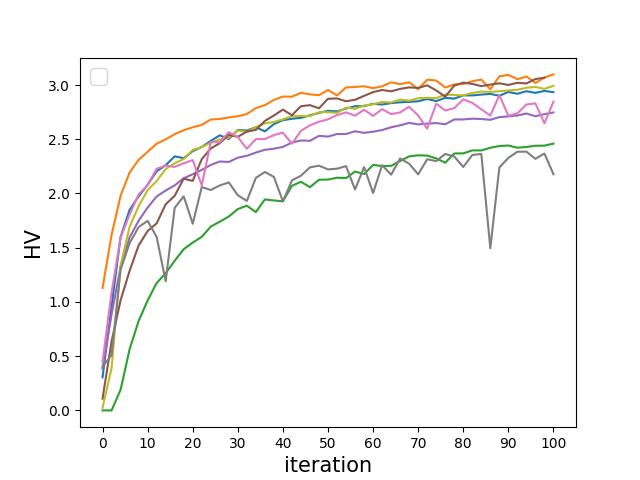}
    \caption{Multi-(Fashion+MNIST)}
  \end{subfigure} 
  \caption{The results for the three MNIST-like datasets. The top row shows the test accuracies above 0.4, the middle row shows the training losses below 2, and the bottom row shows the HV value convergence curves during the training process, calculated using the reference point (2,2).}\label{mnist}
\end{figure*}

The outcomes displayed in Fig.~\ref{mnist} illustrate MT$^2$O's ability to generate multiple well-distributed solutions. The top row exhibits test accuracies, showcasing MT$^2$O's attainment of the highest per-task accuracies in the first two datasets. In the third dataset, Multi-(Fashion+MNIST), where task correlation is lower, MT$^2$O demonstrates compatibility with other methods and the strong single-task baseline. The middle row depicts training losses, MT$^2$O's solutions showcase dominance across all three datasets, indicating an effective optimization process. The bottom row displays HV value curves during training, where MT$^2$O's HV values exhibit marginal superiority or comparability with other methods.


Across all experiments, the performance of LS consistently falls behind other methods. Recall that, in cases where there's no transfer among the subproblems, MT$^2$O, utilizing weighted-sum decomposition, regresses to LS. Consequently, the superior performance observed in MT$^2$O primarily arises from inter-subproblem transfer.
Intuitively, this means that the fitness landscape of the training deep neural networks is likely to be non-convex, leading gradient-based methods to converge on inferior local Pareto optima. Inter-subproblem transfer serves as a means to alleviate this issue.

In MT$^2$O, all the weights of the network for each subproblem are transferred among the subproblems based on the transfer coefficients in the experiment. An interesting topic is to transfer partial weights, such as only weights of the convolution layers are transferred to promote similarities among subproblems. This can be achieved simply by setting the transfer coefficients to a diagonal matrix instead of a scalar. More sophisticated methods can thus be considered to design transfer coefficients that best fit the problem of interest.

\subsubsection{Image classification with many tasks}

To verify our algorithm's efficacy on many task datasets and its adaptability across different base networks, we conduct experiments on the CelebA dataset \cite{liu2015deep} employing ResNet \cite{he2016deep} as the base network. With 200K face images annotated for 40 attributes, each representing a binary classification task, we concentrate on 17 attributes categorized under a common group, as delineated by \cite{sener2018multi}. We adopt the same MTL network architecture utilized in \cite{sener2018multi}, which leverages ResNet-18 without its final layer as a shared representation function and incorporates a linear layer for each attribute.

Handling 17 tasks posed challenges in utilizing reference vectors for scalable subproblem decomposition. To address this, we simplified by assigning unit vectors as references, one for each task, creating 17 individual tasks. Despite this simplification, achieving results with other Pareto MTL methods, except for MOOMTL, proved impractical. MOOMTL, requiring the generation of a single Pareto model, remained feasible. Thus, this section compares our algorithm with MOOMTL and the baseline, which entails training a single model for each task. We determined subproblem distances based on cosine similarities between label vectors and set the neighborhood size $J=2$.

Figure~\ref{celeba_radar} presents a radar chart displaying the percentage of misclassification errors for each binary classification task. Our method demonstrates superior performance over baselines and MOOMTL across a majority of tasks and achieves comparable results in the remaining tasks. This experiment emphasizes the effectiveness of our method, particularly in managing a high number of tasks.

\begin{figure}[!htb]
    \includegraphics[width=0.9\linewidth]{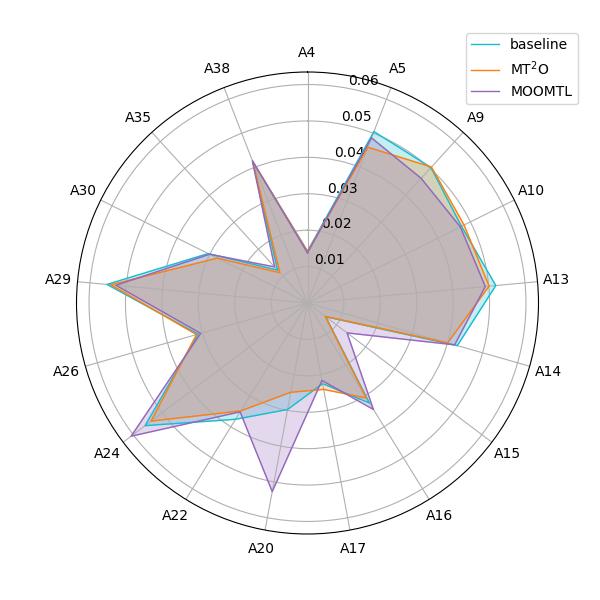}
  \caption{Attribute-wise misclassification error percentage on CelebA dataset. Lower values indicate better performance.}\label{celeba_radar}
\end{figure}

\subsubsection{Multi-Target Regression}
 
We conducted experiments on the River Flow dataset \cite{spyromitros2016multi}, comprising eight tasks aimed at predicting flow patterns over 48 hours at eight sites within the Mississippi River network. Each sample includes recent and time-lagged observations from the eight sites, resulting in 64 features and eight target variables. Training utilized 6,303 samples, with an additional 2,702 samples allocated for testing. Same as \cite{mahapatra2020multi}, we employed a four-layer fully connected feed-forward neural network (FNN) as the MTL model. The model was trained using ten randomly selected reference vectors with Mean Squared Error (MSE) as the loss function. Each of the eight objectives was separately trained as baseline models. 
Given the 8 tasks, we visualized the methods using the relative loss profile (RLP), as presented in \cite{mahapatra2020multi}. Specifically, RLP for each task represents the mean value of scaled losses, scaled by the reference vector of the subproblems. The results are depicted in Fig.~\ref{rf1}. 

\begin{figure}[!hth]
\centering
    \includegraphics[width=0.9\linewidth]{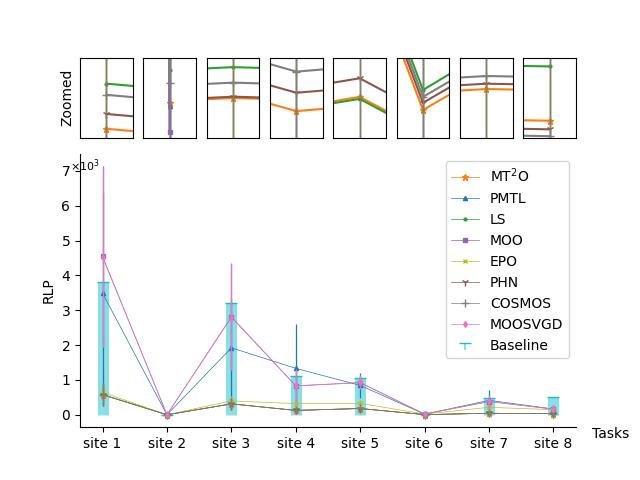}
  \caption{The mean RLP (with standard deviation) for predicting river flow at 8 sites of the Mississippi River. The MSE of the baseline for each site is divided by 8 for a fair comparison.}\label{rf1}
\end{figure}

To facilitate a clearer comparison, the figure on top depicts a zoomed-in view. Our observation reveals that MT$^2$O consistently demonstrates either superior or highly competitive performance compared to other state-of-the-art methods. Specifically, in 5 out of 8 tasks, MT$^2$O outperforms other methods, while displaying slightly inferior performance in the remaining 3 tasks. 

The superior performance of MT$^2$O over the baseline provides evidence for the advantages of MTL in handling correlated tasks. This stands in contrast to the traditional approach of learning each task independently. Furthermore, the improved performance of MT$^2$O in comparison to other MTL methods, especially LS, substantiates the significance of transferring and sharing valuable information across subproblems. This demonstrates the effectiveness of leveraging knowledge from one subproblem to enhance the solution of another, reinforcing the practical value of MTL in optimizing tasks within interrelated problem domains.

\subsubsection{Pixel-Wise Classification and Regression}

We extended our method's evaluation to a more complex scene understanding problem using the NYUv2 dataset \cite{silberman2012indoor}. This challenging indoor scene dataset comprises 1,449 RGBD images with dense per-pixel labeling, focusing on three learning tasks: 13-class semantic segmentation \cite{couprie2013indoor}, depth estimation, and surface normal estimation \cite{eigen2015predicting}. We adopted the experimental setup outlined in \cite{liu2019end}, including image preprocessing, task-specific loss functions, and a base MTL neural network architecture. The MTL model consists of SegNet \cite{badrinarayanan2017segnet} serving as the shared representation encoder and three task-specific lightweight convolutional layers. The dataset used matches that in \cite{liu2019end} \footnote{Downloaded from: https://github.com/lorenmt/mtan}. Our dataset split comprised 796 training images and 654 testing images. The model underwent training for 200 epochs. Hyperparameters remained consistent with previous experiments, except for a batch size of 2. For reference vectors, we employed five randomly generated vectors and three unit vectors. Due to significant time consumption in obtaining acceptable results on the larger NYUv2 dataset, PMTL and PHN are omitted from this experiment. For COSMOS, we adopted $\alpha=1.2$, consistent with the parameters used in the paper for the MultiMNIST dataset. We conducted a search for $\lambda$ within the range [1, 2, 3, 4, 5, 8, 10] and reported the better result for $\lambda=3$.

Fig.~\ref{nyuv2} illustrates the training losses between each pair of tasks, demonstrating MT$^2$O's superior performance compared to all considered methods. This is further corroborated by the HV values shown in Fig.~\ref{nyuv2_hv}, affirming MT$^2$O's capacity to discover better Pareto optimal solutions. Notably, around epoch 100, MT$^2$O achieves an HV value around 14, reaching this benchmark almost 2 times faster than EPO and LS, which approach this value by epoch 200.

Table~\ref{nyuv2result} showcases the testing results for the three tasks. While baseline and MTLMOO lack the utilization of reference vectors' information, we present the testing results of a single run for these models. The best performance scores for each task are highlighted in bold. MT$^2$O consistently outperforms all comparison algorithms for each task. Specifically, the enhancement observed in MTL algorithms (LS, EPO, MT$^2$O) over the baseline is attributed to the implicit transfer within the MTL architecture. MT$^2$O's further improvement over other algorithms is owed not only to the implicit transfer within the MTL architecture but also to the explicit transfer between produced subproblems, emphasizing the advantage of our proposed algorithm over other Pareto MTL algorithms. The underwhelming performance of COSMOS and MOOSVGD likely stems from insufficiently tuned hyperparameters and improper initial values, highlighting the challenge of intricate parameter tuning in these works.

\begin{figure*}[!htb]
\centering
 \begin{subfigure}[b]{0.32\linewidth}
    \includegraphics[width=\textwidth]{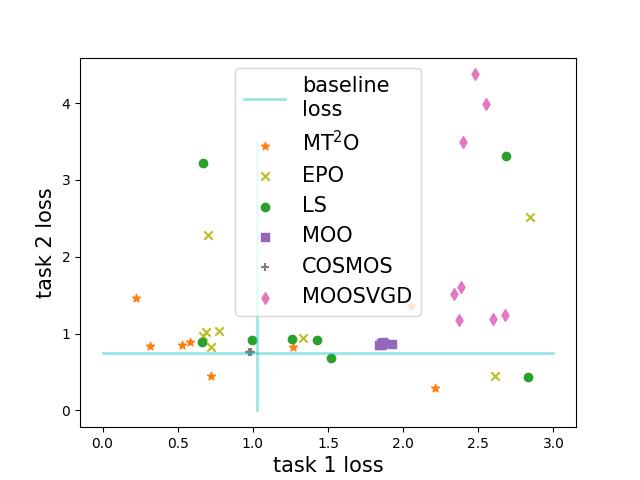}
    \caption{Task 1 vs Task 2}
  \end{subfigure} 
  \begin{subfigure}[b]{0.32\linewidth}
    \includegraphics[width=\textwidth]{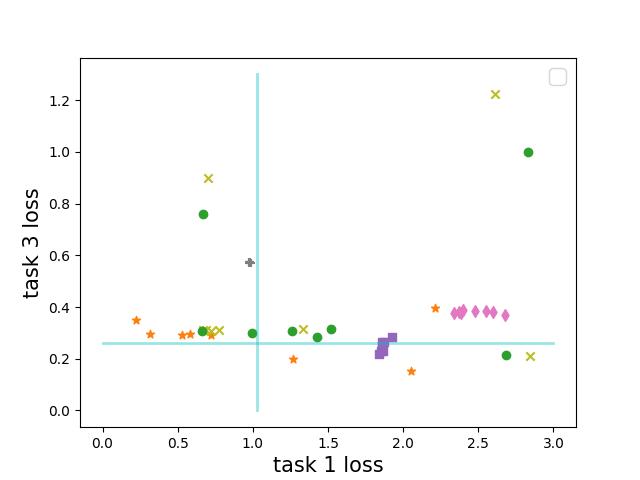}
    \caption{Task 1 vs Task 3}
  \end{subfigure} 
   \begin{subfigure}[b]{0.32\linewidth}
    \includegraphics[width=\textwidth]{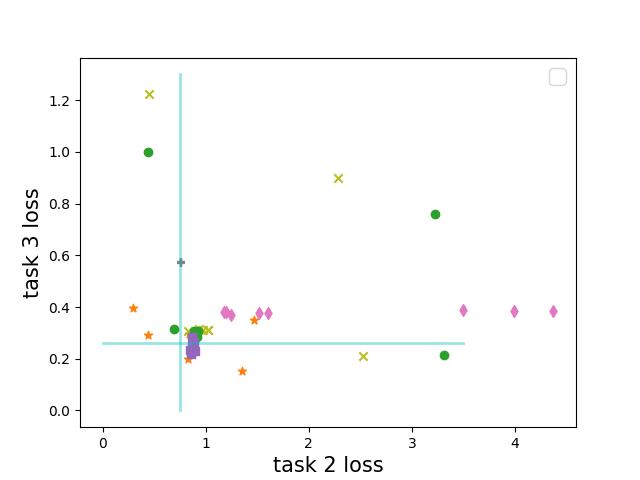}
    \caption{Task 2 vs Task 3}
  \end{subfigure} 
  \caption{The 2-D projections for the results obtained by different algorithms on NYUv2 dataset. We use five randomly generated reference vectors and three unit vectors.}\label{nyuv2}
\end{figure*}

\begin{figure}[!htb]
    \includegraphics[width=0.9\linewidth]{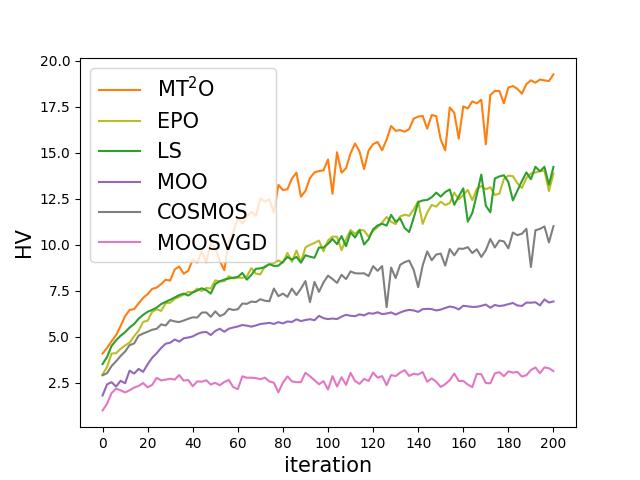}
  \caption{The HV value convergence curves during training process for the NYUv2 dataset, calculated using the reference point (3,3,3).}\label{nyuv2_hv}
\end{figure}

\begin{table*}[!htb]
    \centering
    \normalsize\caption{Semantic segmentation, depth estimation, and surface normal prediction results on the NYUv2 testing dataset. The best testing scores for each task are highlighted in bold.}\label{nyuv2result}  
    \normalsize\centering\begin{tabular*}{\hsize}{@{}@{\extracolsep{\fill}{}}l|c|cccccc@{}}
        \hline\hline
         \multirow{2}{*}{Algorithm} & \multirow{2}{*}{Reference Vector} & \multicolumn{2}{c}{Segmentation $\uparrow$}  & \multicolumn{2}{c}{Depth $\downarrow$} & \multicolumn{2}{c}{\makecell{Surface Normal\\ (Angle Distance) }$\downarrow$} \\
         && mIoU & Pix Acc & Abs Err & Rel Err &  Mean & Media \\
         \hline
         baseline  & - & 0.1587 & 0.4658 & 0.7515 & 0.3007 & 36.7616 & 33.6900 \\
         \hline
         MTLMOO & - & 0.0774 &0.3402 & 0.8536 & 0.3599 & 37.2142 & 34.4094 \\
         \hline
         \multirow{8}{*}{LS} & (0.57368574 0.09172111 0.33459315)& 0.1436 & 0.4321 & 0.8801 & 0.3649 & 41.7440 & 40.0487 \\
        
          & (0.20438813 0.07170215 0.72390972) & 0.1123 & 0.3629 & 0.8047 & 0.3408 & 42.3555 & 39.6812\\
         
          & (0.78867472 0.12564263 0.08568266) & 0.1704 & 0.4689 & 0.8475 & 0.3578 & 42.5818 & 40.9144 \\
          
          & (0.44062379 0.36721539 0.19216082) & 0.1075 & 0.3959 & 0.8707 & 0.3663 
          & 40.9371 & 38.8786 \\
          
          & (0.56127324 0.12731909 0.31140767) & 0.1390 & 0.4199 & 0.8682 & 0.3978 
          & 44.1093 & 43.1565 \\

          & \makebox[4em][s](1\hspace{\fill} 0 \hspace{\fill}0) & 0.1703 & 0.4511 & 3.1798 & 1.2083 
          & 71.8023 & 68.5502 \\

          & \makebox[4em][s](0\hspace{\fill} 1\hspace{\fill} 0) & 0.0200 & 0.0458 & 0.7219 & 0.3100 
          & 92.7729 & 92.6960 \\

          & \makebox[4em][s](0\hspace{\fill} 0\hspace{\fill} 1) & 0.0255 & 0.0760 & 3.1579 & 1.1848 
          & 36.8948 & 34.0400 \\

          \hline
          \multirow{8}{*}{EPO} & (0.57368574 0.09172111 0.33459315)& 0.1683 &0.4635 & 0.9634 & 0.4131 & 43.0733 & 41.4054\\
        
          & (0.20438813 0.07170215 0.72390972) & 0.1640 & 0.4545 & 0.9255 & 0.4322 & 43.1366 & 41.5925 \\
         
          & (0.78867472 0.12564263 0.08568266) & 0.1602 & 0.4461 & 0.9426 & 0.4348 & 44.0143 & 42.9994 \\
          
          & (0.44062379 0.36721539 0.19216082) & 0.1347 & 0.4232 & 0.8780 & 0.3948 
          & 44.5121 & 43.1929 \\
          
          & (0.56127324 0.12731909 0.31140767) & 0.1715 & 0.4578 & 0.9064 & 0.3617 
          & 42.9482 & 41.0180 \\

          & \makebox[4em][s](1\hspace{\fill} 0\hspace{\fill} 0) & 0.1618 & 0.4525 & 2.2595 & 0.8291 
          & 80.9600 & 78.3901 \\

          & \makebox[4em][s](0\hspace{\fill} 1\hspace{\fill} 0) & 0.0280 & 0.1209 & 0.7288 & 0.3031 
          & 106.7012 & 109.4692 \\

          & \makebox[4em][s](0\hspace{\fill} 0\hspace{\fill} 1) & 0.0268 & 0.0674 & 2.4440 & 0.9103 
          & 36.3225 & 33.2950 \\

          \hline

          \multirow{8}{*}{COSMOS} & (0.57368574 0.09172111 0.33459315)& 0.1568 &0.4407 & 0.8250 & 0.3507 & 60.9747 & 57.8004\\
        
          & (0.20438813 0.07170215 0.72390972) & 0.1566 & 0.4406 & 0.8249 & 0.3504 & 61.0166 & 57.8608 \\
         
          & (0.78867472 0.12564263 0.08568266) & 0.1562 & 0.4397 & 0.8258 & 0.3515 & 60.8605 & 57.6475 \\
          
          & (0.44062379 0.36721539 0.19216082) & 0.1563 & 0.4400 & 0.8250 & 0.3500 
          & 61.0006 & 57.8733 \\
          
          & (0.56127324 0.12731909 0.31140767) & 0.1569 & 0.4409 & 0.8252 & 0.3508 
          & 60.9699 & 57.7962 \\

          & \makebox[4em][s](1\hspace{\fill} 0\hspace{\fill} 0) & 0.1564 & 0.4397 & 0.8264 & 0.3521 
          & 60.7405 & 57.5052 \\

          & \makebox[4em][s](0\hspace{\fill} 1\hspace{\fill} 0) & 0.1568 & 0.4410 & 0.8260 & 0.3509 
          & 61.0260 & 57.8876 \\

          & \makebox[4em][s](0\hspace{\fill} 0\hspace{\fill} 1) & 0.1561 & 0.4395 & 0.8255 & 0.3497 
          & 60.8800 & 57.7392 \\

          \hline
          
          \multirow{8}{*}{MOOSVGD} & (0.57368574 0.09172111 0.33459315)& 0.0357 &0.2232 & 1.4351 & 0.5157 & 49.0525 & 48.5495\\
        
          & (0.20438813 0.07170215 0.72390972) & 0.0363 & 0.1834 & 1.1053 & 0.4014 & 49.4508 & 48.9521 \\
         
          & (0.78867472 0.12564263 0.08568266) & 0.0358 & 0.2365 & 3.2810 & 1.2050 & 49.7757 & 49.7691 \\
          
          & (0.44062379 0.36721539 0.19216082) & 0.0305 & 0.2425 & 3.7848 & 1.4375 
          & 49.5902 & 49.3982 \\
          
          & (0.56127324 0.12731909 0.31140767) & 0.0329 & 0.2060 & 1.5260 & 0.5300 
          & 48.9583 & 48.3475 \\

          & \makebox[4em][s](1\hspace{\fill} 0\hspace{\fill} 0) & 0.0279 & 0.1568 & 1.2135 & 0.5756 
          & 48.3452 & 47.5633 \\

          & \makebox[4em][s](0\hspace{\fill} 1\hspace{\fill} 0) & 0.0391 & 0.1991 & 1.1037 & 0.4355 
          & 49.3026 & 48.9025 \\

          & \makebox[4em][s](0\hspace{\fill} 0\hspace{\fill} 1) & 0.0407 & 0.1815 & 4.2289 & 1.6728 
          & 49.4276 & 49.0986 \\

          \hline

           \multirow{5}{*}{MT$^2$O} & (0.57368574 0.09172111 0.33459315)& 0.1747 &0.4440 & 0.8430 & 0.3425 & 41.1900 & 39.6581 \\
        
          & (0.20438813 0.07170215 0.72390972) & 0.1762 & 0.4510 & 0.7382 & 0.2990 & 41.3554 & 39.8060  \\
         
          & (0.78867472 0.12564263 0.08568266) & \textbf{0.2007} & \textbf{0.4842} & 0.8411 & 0.3589 & 41.5276 & 39.9314  \\
          
          & (0.44062379 0.36721539 0.19216082) & 0.1209 & 0.4267 & 0.8194 & 0.3428 
          & 36.3683 & 33.2279  \\
          
          & (0.56127324 0.12731909 0.31140767) & 0.1743 & 0.4461 & 0.8686 & 0.3580 
          & 41.2362 & 39.4539 \\

          & \makebox[4em][s](1\hspace{\fill} 0\hspace{\fill} 0) & 0.2002 & 0.4802 & 1.2034 & 0.4453 
          & 46.0092 & 44.3196 \\

          & \makebox[4em][s](0\hspace{\fill} 1\hspace{\fill} 0) & 0.0434 & 0.2151 & \textbf{0.6793} & \textbf{0.2972} 
          & 47.2681 & 45.7865 \\

          & \makebox[4em][s](0\hspace{\fill} 0\hspace{\fill} 1) & 0.0594 & 0.3062 & 1.2359 & 0.4372 
          & \textbf{35.6208} & \textbf{31.7099}\\

        \hline\hline
    \end{tabular*}
\end{table*}

\section{Conclusion}
In this paper, we have proposed a novel Multi-Task Learning with Multi-Task Optimization (MT$^2$O) algorithm to arrive at a representative subset of Pareto optimized models with different trade-offs among tasks in MTL. MT$^2$O jointly solves multiple subproblems generated by first transforming MTL into MOO, and then decomposing it using a diverse set of weight vectors in objective space. Exploiting the similarities between subproblems, the iterative transfer of parameter values among them is expected to accelerate convergence toward the Pareto front. 

We presented a theorem demonstrating that, under subproblems with differentiable and convex objective functions, and with symmetric transfer coefficients, the convergence rate is faster than solving the subproblems independently under certain conditions. Empirical studies encompassed various learning tasks with varying task numbers, spanning image classification, data regression, and hybrid classification and estimation. Results validate MT$^2$O's advancement of the state-of-the-art in Pareto MTL. Particularly on the extensive NYUv2 dataset, our method achieves almost 2 times faster convergence than the next-best among the state-of-the-art. The outcome encourages future work on integrating multitasking with gradient-based optimization algorithms, enabling one-pass learning of sets of specialized machine learning models.


%

\appendices

\ifCLASSOPTIONcaptionsoff
  \newpage
\fi

\bibliographystyle{IEEEtran}
\bibliography{ParetoMTL_MTO_ref.bib}

\begin{IEEEbiography}[{\includegraphics[width=1in,height=1.25in,clip,keepaspectratio]{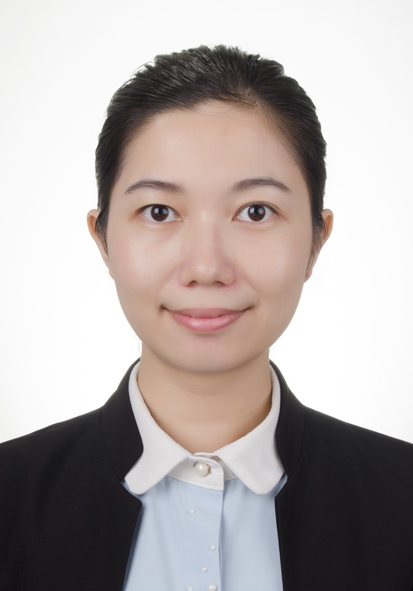}}]{Lu Bai}
received the B.Eng. and M.Eng. degrees in automation from Xiamen University, Xiamen, China, in 2012 and 2015, respectively, and the Ph.D. degree in control and optimization from Nanyang Technological University, Singapore, in 2020. She currently serves as a Research Fellow at Nanyang Technological University, Singapore. Her research interests lie in distributed control and optimization, and multi-task optimization algorithms, with application in multi-task learning and evolutionary multitasking.
\end{IEEEbiography}
\begin{IEEEbiography}[{\includegraphics[width=1in,height=1.25in,clip,keepaspectratio]{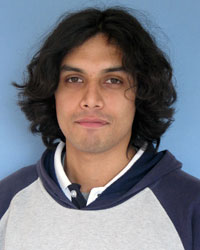}}]{Abhishek Gupta}
received his PhD in Engineering Science from the University of Auckland, New Zealand, in the year 2014. He has diverse research experience in the field of computational science, spanning topics at the intersection of optimization and machine learning, neuroevolution and scientific computing. His current research interests lie in data-efficient search algorithms with transfer and multitask learning capabilities, with application in complex engineering design. He is recipient of the 2019 IEEE Transactions on Evolutionary Computation Outstanding Paper Award for work on evolutionary multitasking. He is associate editor of the IEEE Transactions on Emerging Topics in Computational Intelligence, Complex \& Intelligent Systems journal, and the Memetic Computing journal.
\end{IEEEbiography}
\begin{IEEEbiography}[{\includegraphics[width=1in,height=1.25in,clip,keepaspectratio]{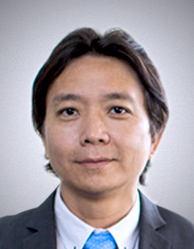}}]{Yew-Soon Ong}
(M’99-SM’12-F’18) received the Ph.D. degree in artificial intelligence in complex design from the University of Southampton, U.K., in 2003. He is President’s Chair Professor in Computer Science at Nanyang Technological University (NTU), and holds the position of Chief Artificial Intelligence Scientist of the Agency for Science, Technology and Research Singapore. At NTU, he serves as Director of the Data Science and Artificial Intelligence Research and co-Director of the Singtel-NTU Cognitive \& Artificial Intelligence Joint Lab. His research interest is in machine learning, evolution and optimization. He is founding Editor-in-Chief of the IEEE Transactions on Emerging Topics in Computational Intelligence and AE of IEEE Transactions on Neural Networks \& Learning Systems, IEEE Transactions on Artificial Intelligence and others. He has received several IEEE outstanding paper awards and was listed as a Thomson Reuters highly cited researcher.
\end{IEEEbiography}







\end{document}